\documentclass[twoside]{article}

\usepackage[accepted]{aistats2025}
\usepackage[utf8]{inputenc} % allow utf-8 input
\usepackage[T1]{fontenc}    % use 8-bit T1 

\usepackage{microtype}
\usepackage{graphicx}
\usepackage{booktabs} % for professional tables
\usepackage{wrapfig}
\usepackage{xspace}

\usepackage{hyperref}
\usepackage{titletoc}

\usepackage{amsmath}
\usepackage{amssymb}
\usepackage{mathtools}
\usepackage{amsthm}
\usepackage[inline]{enumitem}

\newtheorem{assumption}{Assumption}

\usepackage[capitalize,noabbrev]{cleveref}
\usepackage{autonum}

\usepackage{url}
\usepackage{algorithm}
\usepackage{algorithmic}
\usepackage{wrapfig}
\usepackage{tabularx}
\usepackage{adjustbox}
\usepackage{upgreek}
\usepackage{xcolor}
\usepackage{subfigure}
\usepackage{multirow}

\definecolor{lightblue}{cmyk}{0.95,0.0,0.2,0.2}

\newcommand{\ours}{\textsc{DiffOPT}\xspace}
\newtheorem{theorem}{Theorem}

\newcommand{\ie}{\textit{i}.\textit{e}.,}

\newcommand{\wrt}{\textit{w}.\textit{r}.\textit{t}.}
\newtheorem{proposition}{Proposition}
\newcommand{\std}[1]{\scriptsize

{$\pm$#1}}
\newcommand{\bfw}{\mathbf{w}}
\newcommand{\bfy}{\mathbf{y}}
\newcommand{\bfx}{\mathbf{x}}
\newcommand{\bfz}{\mathbf{z}}
\newcommand{\rmd}{\mathrm{d}}
\newcommand{\rset}{\mathbb{R}}
\newcommand{\supp}{\mathrm{supp}}
\newcommand{\rmc}{\mathrm{C}}
\newcommand{\abs}[1]{| #1 |}
\newcommand{\nset}{\mathbb{N}}
\newcommand{\calD}{\mathcal{D}}
\newcommand{\vareps}{\varepsilon}
\newcommand{\Id}{\mathrm{Id}}

\newcommand{\lightblue}[1]{\textcolor{lightblue}{#1}}

\newcommand{\vsra}{{\ensuremath{%
  \mathchoice{\includegraphics[height=1.1ex]{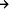}}
    {\includegraphics[height=1.1ex]{symbols/veryshortrightarrow.pdf}}
    {\includegraphics[height=.8ex]{symbols/veryshortrightarrow.pdf}}
    {\includegraphics[height=.5ex]{symbols/veryshortrightarrow.pdf}}
}}}
\newcommand{\vsla}{{\ensuremath{%
  \mathchoice{\includegraphics[height=1.1ex]{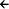}}
    {\includegraphics[height=1.1ex]{symbols/veryshortleftarrow.pdf}}
    {\includegraphics[height=.8ex]{symbols/veryshortleftarrow.pdf}}
    {\includegraphics[height=.5ex]{symbols/veryshortleftarrow.pdf}}
}}}
\usepackage[round]{natbib}

\begin{document}


\runningauthor{Lingkai Kong*, Yuanqi Du*, Wenhao Mu*}

\twocolumn[

\aistatstitle{Diffusion Models as Constrained Samplers for Optimization with Unknown Constraints}

\aistatsauthor{ Lingkai Kong*$^{1}$, \quad Yuanqi Du*$^{2}$, \quad Wenhao Mu*$^{1}$, \quad Kirill Neklyudov$^{3,4}$ } 

\aistatsauthor{Valentin De Bortoli$^{5}$, \quad Dongxia Wu$^{6}$, \quad Haorui Wang$^{1}$,\quad Aaron Ferber$^{2}$}

\aistatsauthor{Yi-An Ma$^{6}$, \quad Carla P. Gomes$^{2}$, \quad Chao Zhang$^{1}$ }

\vspace{0.5em}

\aistatsaddress{$^{1}$Georgia Tech \quad $^{2}$Cornell University \quad $^{3}$Université de Montréal \quad $^{4}$Mila \\$^{5}$ENS, CNRS, PSL
University, $^{6}$ UCSD} ]

\renewcommand{\thefootnote}{\fnsymbol{footnote}}
\footnotetext[1]{Equal contribution, authors agreed ordering can be changed for their respective interests.}
\renewcommand{\thefootnote}{\arabic{footnote}}

\begin{abstract}
Optimization with unknown constraints is a
challenging, yet unsolved problem. Overlooking these constraints can lead to spurious solutions that are unrealistic in practice. To deal with such unknown constraints, we propose to perform optimization within the data manifold using diffusion models. To constrain the optimization
process to the data manifold, we reformulate the original optimization problem
as a sampling problem from the product of the Boltzmann distribution defined
by the objective function and the data distribution learned by the diffusion model.
Depending on the differentiability of the objective function, we propose two different sampling methods. For differentiable objectives, we propose a two-stage
framework that begins with a guided diffusion process for warm-up, followed by a
Langevin dynamics stage for further correction. For non-differentiable objectives,
we propose an iterative importance sampling strategy using the diffusion model as
the proposal distribution. Comprehensive experiments on a synthetic dataset, six
real-world black-box optimization datasets, and a multi-objective molecule optimization dataset show that our method achieves better or comparable performance
with previous state-of-the-art baselines.
\end{abstract}

\vspace{1em}
\section{Introduction}
Optimization problems are ubiquitous in real-world applications when approaching search problems \citep{bazaraa2011linear}, partial differential equations \citep{isakov2006inverse}, molecular design \citep{sanchez2018inverse}. While significant advancements have been made in resolving a broad spectrum of abstract optimization problems with analytically known objective functions and constraints \citep{boyd2004convex,rao2019engineering,petropoulos2023operationsresearchsurvey},  optimization in real-world scenarios remains challenging since the exact nature of the objective is often unknown, and access to constraints is limited \citep{conn2009introduction}. For example, it is challenging to incorporate the closed-form constraints on a molecule to be synthesizable or design an objective function for target chemical properties.

Previous studies have identified problems with unknown objective functions as black-box optimization problems \citep{conn2009introduction,alarie2021two}. In such scenarios, the only way to obtain the objective value is through running a simulation \citep{larson2019derivative} or conducting a real-world experiment \citep{shields2021bayesian}, which might be expensive and non-differentiable. 
A prevalent approach to this challenge involves learning a surrogate model with available data to approximate the objective function which can be implemented in either an online \citep{snoek2012practical,shahriari2015taking,srinivas2010gaussian} or offline manner \citep{COMs, designbench}. 

However, there is a significant lack of research focused on scenarios where analytic constraints are absent. The only works that deal with unknown constraints are from the derivative-free optimization community \citep{audet2004pattern,audet2006mesh, nguyen2023stochastic}. However, these methods can only be applied to simple low-dimensional problems and cannot be applied to more complex problems such as molecule and protein optimization. In practice, overlooking these feasibility constraints during optimization can result in spurious solutions. For instance, the optimization process might yield a molecule with the desired chemical property but cannot be physically synthesized \citep{gao2020synthesizability,du2022molgensurvey}, which would require restarting the optimization from different initializations \citep{krenn2020self,jain2023gflownets}.

To restrict the search space to the set of feasible solutions, we propose to perform optimization within the support of the data distribution or the data manifold.
Indeed, in practice, one usually has an extensive set of samples satisfying the necessary constraints even when the constraints are not given explicitly.
For instance, the set of synthesizable molecules can be described by the distribution of natural products \citep{baran2018natural, vorvsilak2020syba}. To learn the data distribution, we focus on using diffusion models, which recently demonstrated the state-of-the-art performance in image modeling \citep{ddpm, song2020score}, video generation \citep{ho2022imagen}, and 3D synthesis \citep{poole2022dreamfusion}.
Moreover, \citep{pidstrigach2022score,debortoli2022convergence} theoretically demonstrated that diffusion models can learn the data distributed on a lower dimensional manifold embedded in the representation space, which is often the case of the feasibility constraints.

To constrain the optimization process to the data manifold, we reformulate the original optimization problem as a sampling problem from the product of two densities:
i) a Boltzmann density with energy defined by the objective function and ii) the density of the data distribution. The former concentrates around the global minimizers in the limit of zero temperature \citep{hwang1980laplace,gelfand1991recursive}, while the latter removes the non-feasible solutions by yielding the zero target density outside the data manifold.
Depending on whether the objective function is differentiable or non-differentiable, we propose two different sampling methods. When the objective function is differentiable, we propose a two-stage sampling strategy: (i) a guided diffusion process acts as a warm-up stage to provide initialization of data samples on the manifold, and (ii) we ensure convergence to the target distribution via Markov Chain Monte Carlo (MCMC). When the objective is non-differentiable, we propose an iterative importance sampling strategy using diffusion models to gradually improve the proposal distribution. %At each iteration, we apply the forward and backward processes to the particles resampled from the previous iteration, using this as an improved proposal distribution while maintaining diversity.
%We also theoretically prove that the first stage yields a distribution that concentrates on the feasible minimizers in the zero temperature limit.  

The main contributions of this work are: (1) We reformulate the problem of optimization under unknown constraints as a sampling problem from the product of the data distribution and the Boltzmann distribution defined by the objective function.
(2) We propose two different sampling methods with diffusion models, depending on the differentiability of the objective function. (3) Empirically, we validate the effectiveness of our proposed framework on a synthetic toy example, six real-world offline black-box optimization tasks as well as a multi-objective molecule optimization task. We find that our method, named \ours, can outperform state-of-the-art methods.

\begin{figure*}[t]
   \centering
    \subfigure[]{\includegraphics[width=0.4\textwidth]{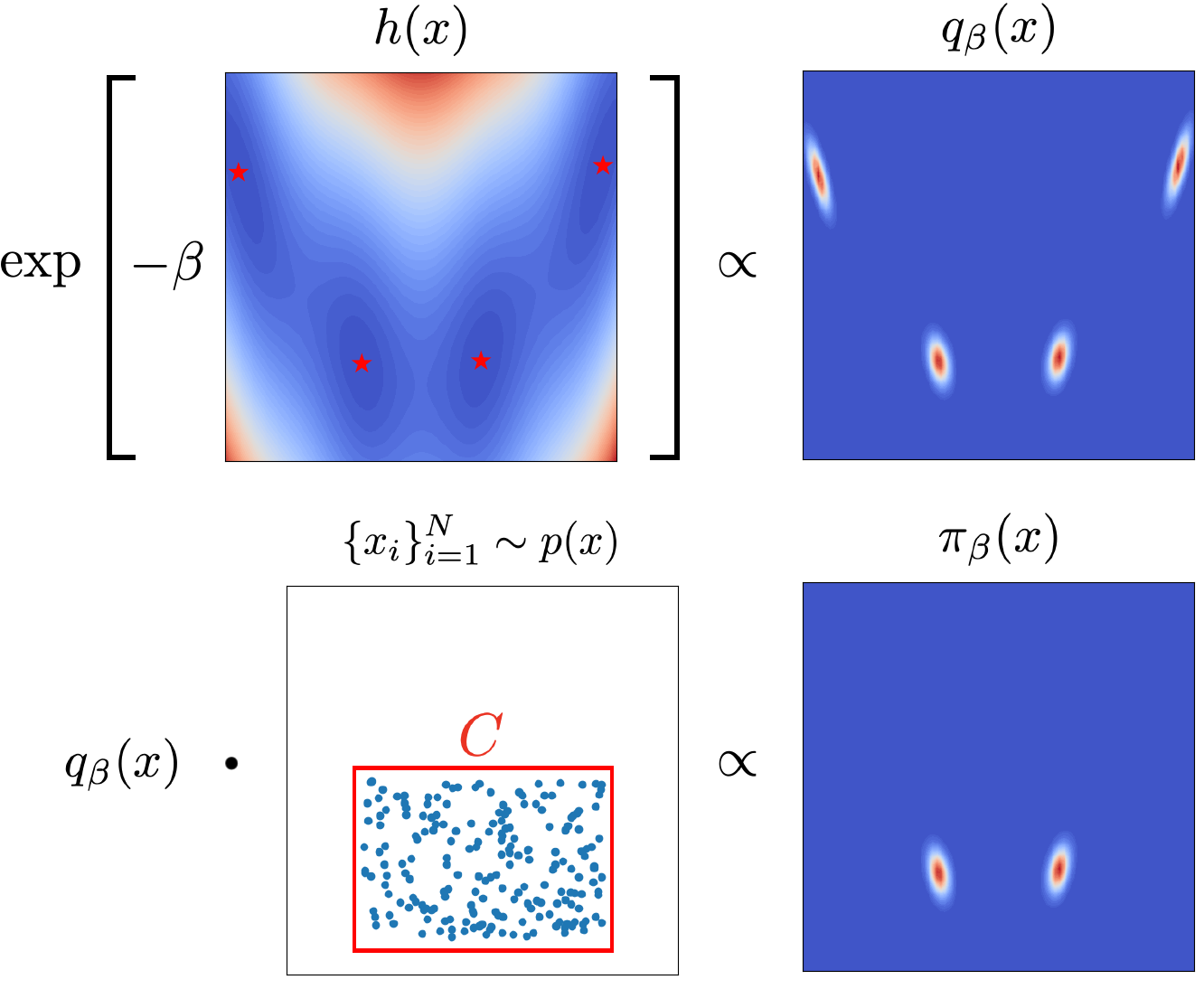}} \subfigure[]{\includegraphics[width=0.4\textwidth]{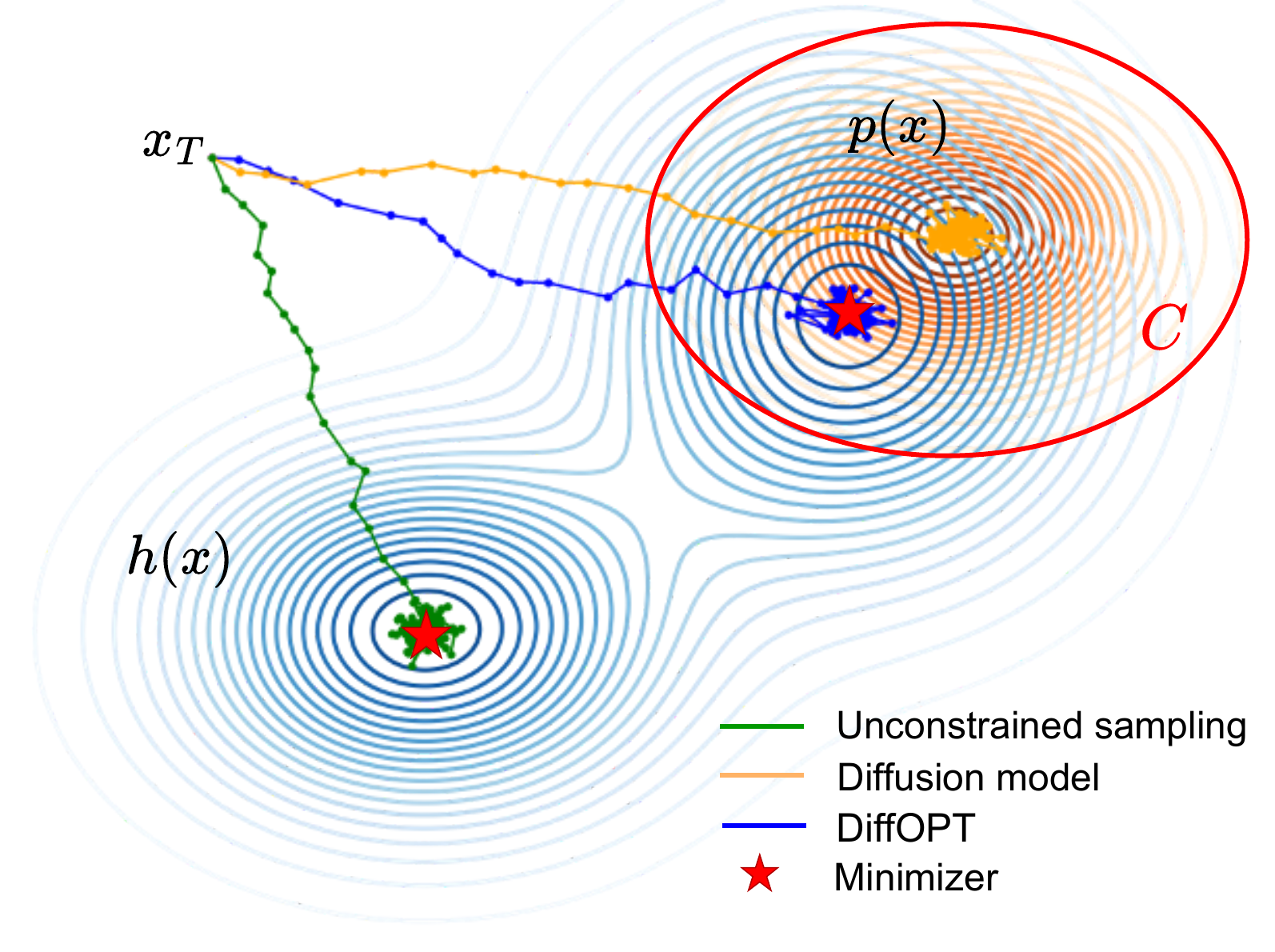}} 
    %\vspace{-0.5em}
    \caption{Constrained optimization as a sampling from the product of densities. That is, we minimize the objective function $h(x)$ (red stars denote the minimizers) within the feasible set $C$, which is given by samples $\{x_i\}_{i=1}^N \sim p(x)$. This problem is equivalent to sampling from the density $\pi_\beta(x) \propto p(x) \exp[-\beta h(x)]$, which concentrates around minimizers of $h(x)$ within the feasible set $C$. The distribution we sample from is shown on the left and the trajectory we take to sample is shown on the right.}
    \label{fig:main}
\end{figure*}

%\vspace{-1em}
\section{Background}
%\vspace{-1em}
\textbf{Problem definition.} Consider an optimization problem with \emph{objective function} $h: \mathbb{R}^d \rightarrow \mathbb{R}$. Additionally, we consider a \emph{feasible set} $C$, which is a subset of $\mathbb{R}^d$. 

Our goal is to find the set of minimizers $\{ x^{\star}_i\}_{i=1}^{M}$ of the objective  $h$ within the feasible set $C$. 
This can be expressed as the following constrained optimization problem
$$
\textstyle \{x^{\star}_i\}_{i=1}^{M} = \arg \min _{x \in C} h(x).
$$
However, in our specific scenario, the explicit formulation of the feasible set $C$ is \emph{unavailable}. 
Instead, we can access a set of points $\mathcal{D}=\{ x_i\}_{i=1}^N$ sampled independently from the feasible set $C$.

\textbf{Optimization via sampling.} If we do not consider the constraints, under mild assumption, the optimization process is equivalent to sample from a Boltzmann distribution $q_{\beta}(x) \propto \exp[-\beta h(x)]$, in the limit where $\beta \to \infty$, where $\beta$ is an \emph{inverse temperature} parameter. This is the result of the following proposition which can be found in  \citep[Theorem 2.1]{hwang1980laplace}.
\begin{proposition}
\vspace{-0.1mm}
\label{prop:laplace_minimizers}
Assume that $h \in \rmc^3(\rset^d, \rset)$. Assume that $\{x_i^\star\}_{i=1}^M$ is the set of minimizers of $h$. Let $p$ be a density on $\rset^d$ such that there exists $i_0 \in \{1, \dots, M\}$ with $p(x_{i_0}^\star) > 0$. Then $Q_\beta$ the distribution with density w.r.t the Lebesgue measure $\propto q_\beta(x)p(x)$ weakly converges to $Q_\infty$ as $\beta \to \infty$ and we have that 
\begin{equation}
    \textstyle Q_\infty = \sum_{i=1}^M a_i \updelta_{x_i^\star} / \sum_{i=1}^M a_i , 
\end{equation}
with $a_i = p(x_i^\star) \det(\nabla^2 h(x_i^\star))^{-1/2}$.
\end{proposition}

Based on this proposition,  \citep{gelfand1991recursive} proposed a \emph{tempered} method for global optimization.
However, the proposed temperature schedule scales logarithmically with the number of steps; hence, the total number of iterations scales exponentially, hindering this method's straightforward application.
In practice, we can sample from the density with the target high $\beta$, see  \citep{raginsky2017non,ma2019sampling,debortoli2021quantitative}. 

\textbf{Product of Experts.}
To constrain our optimization procedure to the feasible set $C$, we propose to model the target density as a product of experts  \citep{hinton2002training}, a modeling approach representing the ``unanimous vote'' of independent models. Given the density models of $m \in \nset$ "experts" $\{q_i(x)\}_{i=1}^m$, the target density of their product is defined as
\begin{align}
\textstyle
    \pi(x) \propto \prod_{i=1}^m q_i(x)\,.
\end{align}
Hence, if one of the experts yields zero density at $x$, the total density of the product at $x$ is zero.

\textbf{Diffusion Models.} 
Given a dataset $\calD = \{x_i\}_{i=1}^N \sim p_{\text{data}}(x)^{\otimes N}$ with $p_{\text{data}}$ concentrated on the feasible set $C$, we will learn a generative model $p$ such that $p \approx p_{\text{data}}$ using a diffusion model  \citep{sohl2015deep,song2019generative,ddpm}. The density $p$ will then be used in our product of experts model to enforce the feasibility constraints, see \Cref{sec:constrained_opt_product}.
In diffusion models, we first simulate a \emph{forward} noising process starting from data distribution $p_0(x)=p_{\text{data}}(x)$ which converges to the standard Gaussian distribution $p_{T}(x)\approx \mathcal{N}(\textnormal{0},\Id)$ as $T \rightarrow \infty$. 
The forward process is defined by the following SDE
\begin{equation}
\textnormal{d}\bfx^{\vsra}_t = f(\bfx^{\vsra}_t, t)\textnormal{d}t + g(t)\textnormal{d}\mathbf{w}_t, \; \bfx^{\vsra}_0\sim p_{\text{data}}(x), \; 0 \leq t \leq T\,,
\end{equation}
where $f:\mathbb{R}^d \rightarrow \mathbb{R}^d$ is a vector-valued drift function, $g(t):  \mathbb{R}\rightarrow \mathbb{R}$ is a scalar-valued diffusion coefficient, and $(\bfw_t)_{t \geq 0}$ is a $d$-dimensional Brownian motion. 
Then, under mild assumptions, the reverse process that generates data from normal noise follows the backward SDE \citep{haussmann1986time, anderson1982reverse}
\begin{align}
\textnormal{d}\bfx^{\vsla}_t =  \left[- f(\bfx^{\vsla}_t, \tau) + g^2(\tau)\nabla_{x} \log p_{\tau}(\bfx^{\vsla}_t)\right]\textnormal{d}t + g(\tau)\textnormal{d}\mathbf{w}_t, 
\vspace{3mm}
\end{align}
where $\tau = T-t$ and $\nabla_{x} \log p_{\tau}(x)$ is the score function which is modeled by a time-dependent neural network via the score matching objective
\begin{align}
\label{eq:sm_objective}
\mathbb{E}_{t} \left[\lambda(t)\mathbb{E}_{\bfx_0^\vsra}\mathbb{E}_{\bfx_t^\vsra|\bfx_0^\vsra}\left[ \|s_\theta(\bfx_t^\vsra, t) - \nabla_{x} \log p_{t|0}(\bfx_t^\vsra|\bfx_0^\vsra)\|^2_2 \right]\right],
\vspace{3mm}
\end{align}
where $p_{t|0}(\bfx_t^\vsra|\bfx_0^\vsra)$ is the conditional density of the forward SDE starting at $\bfx_0^\vsra$, and $\lambda(t) > 0$ is a weighting function. Under assumptions on $f$ and $g$, there exists $\alpha, \sigma$ such that for any $t \in [0, T]$, $\bfx_t^\vsra = \alpha_t \bfx_0^\vsra + \sigma_t \vareps$, where $\vareps \sim \mathcal{N}(0, \Id)$, and therefore one does not need to integrate the forward SDE to sample $(\bfx_0^\vsra, \bfx_t^\vsra)$. 
Recent works have shown that diffusion models can detect the underlying data manifold supporting the data distribution \citep{pidstrigach2022score,debortoli2022convergence}. This justifies the use of the output distribution of a diffusion model as a way to identify the feasible set.

%\vspace{-1.2em}
\section{Proposed Method}
%\vspace{-1.2em}
In this section, we present our method, \ours. First, we formulate the optimization process as a sampling problem from the product of the data distribution  concentrated on the manifold and the Boltzmann distribution defined by the objective function. Then we propose two sampling methods with diffusion models for different types of objective functions.

\subsection{Constrained Optimization as Sampling from Product of Experts}
\label{sec:constrained_opt_product}
We recall that $q_{\beta}(x) \propto \exp[-\beta h(x)]$, where $h$ is the objective function. 
While it is possible to sample directly from $q_{\beta}(x)$, the generated samples may fall outside of the feasible set $C$, defined by the dataset $\calD$. To address this, we opt to sample from the product of the data distribution and the Bolzmann distribution induced by the objective function. Explicitly, our goal is to sample from a distribution $\pi_\beta(x)$ defined as
\begin{align}
\textstyle \pi_\beta(x) \propto p(x)  q_{\beta}(x)\label{eq:target_density}.
\end{align}
Using \Cref{prop:laplace_minimizers}, we have that $\pi_\beta$ concentrates on the \emph{feasible} minimizers of $h$ as $\beta \to + \infty$.

It should be noted that \cref{eq:target_density} satisfies the following properties: (a) it assigns high likelihoods to points $x$ that simultaneously exhibit sufficiently high likelihoods under both the base distributions $p(x)$ and $q_\beta(x)$; (b) it assigns low likelihoods to points $x$ that display close-to-zero likelihood under either one or both of these base distributions. This ensures that the generated samples not only remain within the confines of the data manifold but also achieve low objective values.

% Extending our framework to the multi-objective optimization setting is straightforward. Suppose we aim to maximize a set of objectives, denoted as $\{h^1, h^2, \ldots\}$. In this context, the optimization process can be represented by the distribution $\pi(x)$, which is proportional to the product of the data distribution and the individual Boltzmann distributions for each objective, expressed as $\pi(x) \propto p(x) \prod_i  q^i_{\beta}(x)$. 

In practice, the objective function $h$ can be either differentiable or non-differentiable. In the following sections, we will propose two sampling methods using diffusion models, for each type respectively.

\subsection{Differentiable Objective: Two-stage Sampling with Optimization-guided Diffusion}
\label{sec:two-stage}

Under mild assumptions, the following SDE converges to $\pi_\beta$ \wrt~ the total variation distance \citep{roberts1996exponential}
\begin{align}
    \textnormal{d}\mathbf{x}_t = \nabla_{\mathbf{x}}\log\pi_\beta(\mathbf{x}_t)\textnormal{d}t + \sqrt{2}\textnormal{d}\mathbf{w}_t.
\label{eq:langevin}
\end{align}
where the gradient of the unnormalized log-density can be conveniently expressed as the sum of the scores, \ie~ $ \nabla_x\log \pi_\beta(x) = \nabla_x\log p(x) + \nabla_x\log q_\beta(x)$. Furthermore, one can introduce a Metropolis-Hastings (MH) correction step to guarantee convergence to the target distribution when using discretized version of \cref{eq:langevin} \citep{durmus2022geometric}, which is known as the Metropolis-Adjusted-Langevin-Algorithm (MALA) \citep{grenander1994representations}.
We provide more details of both algorithms in 
% the supplementary material. 
Appendix \ref{sec:mala}.

Theoretically, sampling from \cref{eq:target_density} can be done via MALA. However, in practice, the efficiency significantly depends on the choice of the initial distribution and the step size schedule. The latter is heavily linked with the Lipschitz constant of $\log q_\beta$ which controls the stability of the algorithm. Large values of $\beta$, necessary to get accurate minimizers, also yield large Lipschitz constants which in turn impose small stepsizes. Moreover, the gradient of the log-density can be undefined outside the feasibility set $C$.

To circumvent these practical issues, we propose sampling in two stages: a warm-up stage and a sampling stage.
The former aims to provide a good initialization for the sampling stage. The sampling stage follows the Langevin dynamics for further correction. The pseudocode for both stages is provided in \cref{sec:two-alg}.

\textbf{Stage I: Warm-up with guided diffusion.}
In imaging inverse problems, it is customary to consider \emph{guided} diffusion models to enforce some external condition, see \citep{chung2022diffusion,chung2022improving,song2022pseudoinverse,wu2023practical}. In our setting, we adopt a similar strategy where the guidance term is given by $\beta h$, \ie~we consider
\begin{align}
\begin{split}
\label{eq:warm_up_process_main}
\textnormal{d}\bfx^{\vsla}_t  = & [ -f(\bfx^{\vsla}_t, \tau)  + g^2(\tau)s_{\theta}(\bfx^{\vsla}_t,\tau) -\beta \nabla h({\bfx}_t^\vsla)
] \textnormal{d}t \\ &+ g(\tau)\textnormal{d}\mathbf{w}_t,  \quad \tau=T-t.
\end{split}
\end{align}
\begin{theorem}
\label{thm:concentration_warmup}
    Under assumptions on $p$, $h$, we denote $p_{T-t}^\beta$ the distribution of \cref{eq:warm_up_process_main} at time $t$ and there exists $C > 0$ such that for any $x \in \rset^d$
    \begin{equation}
        (1/C) \tilde{p}_0^\beta(x) \leq p_0^\beta(x) \leq C \tilde{p}_0^\beta(x) ,
    \end{equation}
    where $p_0^\beta$ is the output of the warm-up guided diffusion process and 
        $\tilde{p}_0^\beta(x) = p_0(x) \exp[\log(\beta_0) W_0^\beta(x)] ,$  
    with $W_0^\beta(x) = \Delta h(x) + \langle \nabla \log p_0^\beta(x), \nabla h(x) \rangle$ and $\beta_0$ is the inverse temperature at the end of the process
\end{theorem}
%\vspace{-3em}
The proof is postponed to 
% supplementary materials.
Appendix \ref{sec:appendix_end_distribution}. 
% In \Cref{thm:concentration_warmup}, we show that the output distribution of the warm-up process is upper and lower bounded by a product of experts comprised of 
% \begin{enumerate*}[label=(\roman*)]
%     \item $p_0$ which ensures the \emph{feasibility} conditions
%     \item $\exp[\log(\beta_0) W_0^\beta]$ related to the \emph{optimization} of the objective. 
% \end{enumerate*}
As an immediate consequence of \Cref{thm:concentration_warmup}, we have that $\lim_{\beta_0 \to \infty} p_0^\beta(x_\star) = +\infty$ for every local strict minimizer $x^\star$ of $h$ \emph{within} the support of $p_0(x)$, and $\lim_{\beta_0 \to \infty} p_0^\beta(x_\star) =0$ for $x^\star$  \emph{outside} the support of $p_0(x)$, 
see \Cref{sec:appendix_end_distribution}.
That is, $p_0^\beta$ concentrates on the \emph{feasible} local minimizers of $h$ as $\beta_0 \to +\infty$.

\Cref{thm:concentration_warmup} indicates that the guided diffusion process in the first stage yields a more effectively initialized distribution within the data manifold, bounded both above and below by a product of experts related to the original constrained optimization problem. However, it is known that guided diffusion cannot accurately sample from the product of experts $q_\beta$ \citep{du2023reduce,garipov2023compositional}, see more details in
% supplementary materials.
\Cref{sec:guided-prod}. 
While additional contrastive training of a surrogate objective has been proposed \citep{lu2023contrastive}, in this work, we do not consider such complex corrections. Instead, we rely on ideas from MCMC literature to ensure convergence.

%Sequential Monte Carlo corrections have been proposed to correct this behavior \citep{cardoso2023bayesian,wu2023practical}

\textbf{Stage II: Further correction with Langevin dynamics.}  In the second stage, we further use Langevin dynamics for accurate sampling from $\pi_{\beta}(x)$. The gradient of the log-density of the data distribution $\nabla\log p(x)$ can be obtained by setting the time of the score function to 0, \ie~$s_{\theta}(x, 0)$. The unadjusted Langevin algorithm is then given by
\begin{align}
    \textstyle x^{k+1} = x^k+ (s_{\theta}(x^k,0) - \beta\nabla_{x}h(x^k))\Delta t +\sqrt{2\Delta t}z,
    \vspace{-0.5em}
\end{align}
where $\Delta t$ is the step size and $z$ comes from a Gaussian distribution. In practice, we find that a constant $\beta$ is enough for this stage. When using the score-based parameterization $s_{\theta}(x,t)$, we cannot access the unnormalized log-density of the distribution. Therefore, we cannot use the MH correction step. Although the sampler is not exact without MH correction, it performs well in practice.

To further incorporate the MH correction step, we can adopt an energy-based parameterization following \citep{du2023reduce}:
$E_{\theta}(x, t) = -\frac{1}{2}\|\text{NN}_{\theta
}(x,t)\|^2$, where $E_{\theta}(x, t)$ is the energy function of the data distribution, and $\text{NN}_{\theta}(x,t)$ is a vector-output neural network. An additional benefit of MH correction is that it can enforce hard constraints. The vanilla Langevin dynamics includes a Brownian motion term, whose support spans the entire space (similar to a Gaussian distribution). As a result, samples may occasionally fall outside the constrained region. However, with a MH correction step, the acceptance ratio $p(x^k+1)/p(x^k)$ becomes zero if the proposed sample $x^{k+1}$ lies outside the support of the distribution, ensuring that such samples are disregarded.
% If $x^{k+1}$ is outside the data distribution ($p(x) = 0$), this point cannot be accepted. 
We provide more details of this energy-based parameterization in
% supplementary materials.
\Cref{section:energy}. 

%An additional benefit of MH correction is one can automatically tune the hyperparameters, such as the step size according to the acceptance rate \citep{du2023reduce}.

% \textbf{Extension to black-box objective.}
% Our framework can be easily adapted to black-box optimization. In such a case,  we utilize a logged dataset, denoted as $\{x_i, y_i\}_{i=1}^N$, where each $y_i$ represents the objective value evaluated at the point $x_i$. This dataset allows us to train a surrogate objective function $h_{\phi}(x)$ using deep neural networks.
%\vspace{-1.2em}
\subsection{Non-differentiable Objective: Iterative Importance Sampling with Diffusion Models}
\vspace{-0.5em}

% \Cref{sec:two-stage} provides a sampling that requires the gradient of the objective function. However, in many real-world optimization problem, we may not have access to the gradient.

% Despite the promise of gradient-based optimization, we do not have access to derivatives in many real-world optimization problem which spurs a series of derivative-free optimization method such as genetic algorithm, simulated annealing and zeroth-order optimization \citep{conn2009introduction}. 
When gradients of the objective function are unavailable, we can use self-normalized importance sampling (SNIS) \citep{rubinstein2016simulation}, which is easy and fast to implement. SNIS first proposes several particles from a proposal distribution. Then, resample the particles according to their weights, calculated as the ratio of the target density to the proposal density.

However, the performance of SNIS is heavily determined by the proposal distribution. A good proposal distribution should be close to the target distribution. To address this, we employ an iterative importance sampling with diffusion models to improve the initial proposal distribution. The pseudocode of our derivative-free sampling algorithm is provided in 
% supplementary materials.
\Cref{sec:snis}.

 \textbf{Initialization with SNIS.} We begin by randomly sampling \( S \) particles \(\{x_s^0\}_{s=1}^S\) from the diffusion model \( p_{\theta}(x) \), assigning each particle a weight \( w_s^0 = \frac{\pi_{\beta}(x_s^0)}{p_{\theta}(x_s^0)} = q_{\beta}(x_s^0) \). We then resample \( S \) particles based on the normalized weights \(\tilde{w}_s^0 = \frac{w_s^0}{\sum_{j=1}^S w_j^0} \) using multinomial 
 sampling with replacement.

\textbf{Diffusion-guided proposal for iterative importance sampling.} At the $k$-th iteration, we have \(S\) particles \(\{x_s^{k-1}\}_{s=1}^S\) resampled from the previous iteration. The key insight of our method is that these resampled particles form a proposal distribution closer to the target distribution. 
However, these resampled particles tend to be repetitive and lack diversity. To address this, we use a diffusion model to diversify the resampled particles, ensuring that the proposal distribution remains diverse and closely approximates the target distribution. 
Henceforth, we will use $x_{s}^k$ and $x_{s,0}^k$ interchangeably to refer to the particle at time $0$. We also omit the arrows denoting forward/reverse process of diffusion models for simplicity.

The diversify process works as follows: we first add noise to the particles \(x_{s,0}^{k-1}\) through a forward diffusion process until a randomly sampled time \(t\) to get \(x_{s,t}^{k-1} \sim p_{t|0}(x_{s,t}^{k-1} | x_{s,0}^{k-1})\). Then we  denoise them  back to obtain new particles \(x_{s,0}^{k} \sim p_{0|t}(x_{s,0}^{k} | x_{s,t}^{k-1})\). This proposal can be written as \(Q(x_s^{k} | x_s^{k-1}) = \int p_{t|0}(x_{s,t}^{k-1} | x_{s,0}^{k-1}) p_{0|t}(x_{s,0}^{k} | x_{s,t}^{k-1}) \, d{x_{s,t}^{k-1}}\). We use the Monte Carlo method to compute the marginalization by drawing \(J\) samples \(x_{s,t,j}^{k-1}\) for each particle: \(Q(x_s^{k} | x_s^{k-1}) \approx \frac{1}{J} \sum_{j=1}^J p_{0|t}(x_{s,0}^{k} | x_{s,t,j}^{k-1})\). Sampling \(x_{s,t,j}^{k-1}\) is straightforward as the forward process in diffusion models admits a closed-form conditional Gaussian distribution: \(p_{t|0}(x_{s,t}^{k-1} | x_{s,0}^{k-1}) = \mathcal{N}(x_{s,t}^{k-1} | \alpha_t x_{s,0}^{k-1}, \sigma_t^2\Id)\). Following \citep{song2023loss}, we approximate \(p(x_{s,0}^{k}|x_{s,t}^{k-1})\) as $\mathcal{N}(\frac{x^{k-1}_{s,t,j} + \sigma_t s_{\theta}(x_{s,t,j}^{k-1}, t)}{\alpha_t}, \frac{\sigma^2_t}{1+\sigma_t}\Id )$. The marginalization $Q(x_s^{k})$ is computed by enumerating all the particles from last iteration, \ie~$\frac{1}{S}\sum_{s'=1}^SQ(x_s^{k}|x_{s'}^{k-1})$. Each particle is assigned a weight \( w_s^{k} = \frac{\pi_{\beta}(x_s^{k})}{Q(x_s^{k})}\). We then resample \(S\) particles based on the normalized weights. This process is iterated for \(K\) steps. 

Note that by involving only resampling and diffusion models in each iteration, we can ensure staying within the data manifold, thereby satisfying hard constraints.

\section{Related Work}

\textbf{Diffusion models and data manifold.}
Diffusion models have demonstrated impressive performance in various generative modeling tasks, such as image generation \citep{song2020score, ddpm}, video generation \citep{ho2022imagen}, and protein synthesis \citep{watson2023novo, gruver2023protein}. 
Several studies reveal diffusion models implicitly learn the data manifold \citep{pidstrigach2022score, de2021diffusion,du2023flexible,wenliang2023score}. This feature of diffusion models has been used to estimate the intrinsic dimension of the data manifold in \citep{stanczuk2022your}. Moreover, the concentration of the samples on a manifold can be observed through the singularity of the score function. This phenomenon is well-understood from a theoretical point of view and has been acknowledged in \citep{debortoli2022convergence,chen2022sampling}.

\textbf{Optimization as sampling problems.}
Numerous studies have investigated the relationship between optimization and sampling \citep{ma2019sampling, stephan2017stochastic, trillos2023optimization, wibisono2018sampling, cheng2020interplay, cheng2020stochastic}. Sampling-based methods have been successfully applied in various applications of stochastic optimization when the solution space is too large to search exhaustively \citep{laporte1992vehicle} or when the objective function exhibits noise \citep{branke2004sequential} or countless local optima \citep{burke2005robust, burke2020gradient}. A prominent solution to global optimization is through sampling with Langevin dynamics \citep{gelfand1991recursive}, which simulates the evolution of particles driven by a potential energy function. Furthermore, simulated annealing \citep{kirkpatrick1983optimization} employs local thermal fluctuations enforced by Metropolis-Hastings updates to escape local minima \citep{metropolis1953equation, hastings1970monte}. More recently, \citep{zhang2023let} employs generative flow networks to amortize the cost of the sampling process for combinatorial optimization with both closed-form objectives and constraints.

\textbf{Learning for optimization.}
Recently, there has been a growing trend of adopting machine learning methods for optimization tasks. The first branch of work is model-based optimization, which focuses on learning a surrogate model for the black-box objective function. This model can be developed in either an online \citep{snoek2012practical, shahriari2015taking, srinivas2010gaussian,zhang2021unifying}, or an offline manner \citep{yu2021roma, trabucco2021conservative, fu2020offline, chen2023bidirectional, yuan2023importance}. Additionally, some research \citep{KumanMINs, DDOM, kim2023bootstrapped} has explored the learning of stochastic inverse mappings from function values to the input domain, utilizing generative models such as generative adversarial nets \citep{goodfellow2014generative} and diffusion models \citep{song2020score, ddpm}.

The second branch, known as ``Learning to Optimize'', involves training a neural network to address fully specified optimization problems. In these works, a model is trained using a distribution of homogeneous instances, to achieve generalization on similar unseen instances. Various learning paradigms have been used in this context, including supervised learning \citep{li2018combinatorial, gasse2019exact}, reinforcement learning \citep{li2016learning, khalil2017learning, kool2018attention}, unsupervised learning \citep{karalias2020erdos, wang2022unsupervised,min2023unsupervised}, and generative modeling \citep{sun2023difusco,li2023distribution}. 
%such as improving optimization solvers \citep{bengio2021machine}, end-to-end learning with optimizers \citep{kotary2021end}, and reinforcement learning \citep{mazyavkina2021reinforcement}. 
In contrast to our approach, these works typically involve explicitly defined objectives and constraints.

\textbf{Optimization under unknown constraints.} Existing work on optimization under unknown constraints are based on derivative-free optimization methods, such as generalized pattern search \citep{audet2004pattern}, mesh adaptive direct search \citep{audet2006mesh}, line search \citep{fasano2014linesearch, liuzzi2016derivative}, the Frank-Wolfe algorithm \citep{usmanova2019safe}, and stochastic zeroth-order constraint extrapolation \citep{nguyen2023stochastic}.

However, these works differ from ours and are not directly comparable. First, the settings are different---they require oracle evaluation of whether a proposed solution violates constraints during the optimization process, whereas we do not need this. We only require data samples from the feasible set. Second, these works are hardly applied to the high-dimensional problems that we focus on, such as molecule optimization, protein design, and robot morphology optimization.

% focus more on theoretical analysis and are typically applied to simple optimization problems with small search spaces. In contrast, we focus on much more challenging problems where the search space is extremely large. To the best of our knowledge, none of these methods can be effectively applied to our applications, such as molecule optimization, protein optimization, and robot morphology optimization.

\begin{figure*}[t]
    \centering
    \begin{center}
        \includegraphics[width=0.9\textwidth]{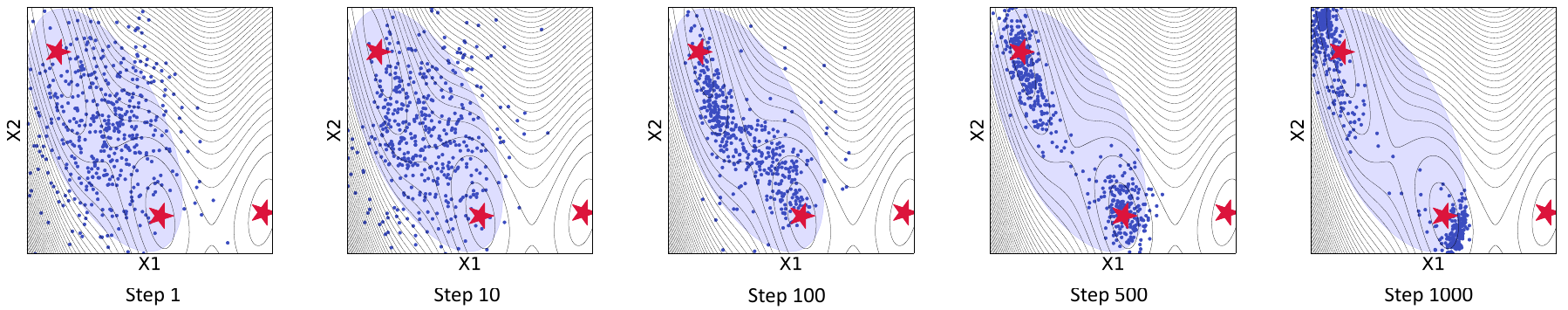}
    \end{center}
    %\vspace{-1em}
    \caption{Sampling trajectory of DiffOPT in the synthetic Branin experiment with unknown constraints. 
    Red stars denote the minimizes, and the blue region denotes the feasible space from which training data is sampled. \ours can effectively navigate towards the two feasible minimizers.}
    \label{fig:landscape_unknown}
\end{figure*}

\begin{figure*}[t]
    \centering
    \includegraphics[width=0.9\textwidth]{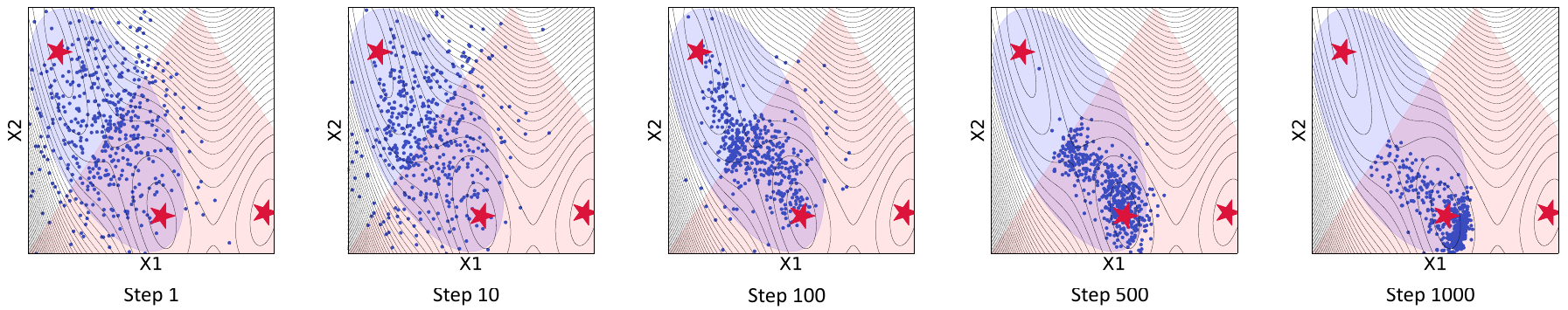}
    %\vspace{-1.2em}
    \caption{Sampling trajectory of DiffOPT in the synthetic Branin experiment with additional known constraints. Red stars denote the minimizers, the blue region denotes the feasible space from which training data is sampled and the pink region denotes the feasible space defined by the added given constraints. \ours can effectively navigate towards the unique minimizer at the intersection of the two feasible spaces.}
    \label{fig:landscape_partial_unknown}
\end{figure*}

\section{Experiment}
%\vspace{-1em}

\begin{table*}[t]
\centering
\resizebox{\textwidth}{!}{
\begin{tabular}{c c c c c c c c}
\bottomrule[1.5pt]
\addlinespace[5pt]
\vspace{5pt}
\textbf{Baseline}     & \textbf{TFBind8} & \textbf{TFBind10} & \textbf{Superconductor} & \textbf{Ant}      & \textbf{D'Kitty}  & \textbf{ChEMBL}   &\textbf{Mean Rank}   \\
\hline \addlinespace[5pt]
Dataset Best       & $0.439$     & $0.00532$    & $74.0$           & $165.326$  & $199.231$  & $383.7e3$  &- \vspace{3pt}\\
\hline \addlinespace[5pt]
CbAS            & $0.958$\std{0.018}     & $0.761$\std{0.067}      & $83.178$\std{15.372}         & $468.711$\std{14.593}  & $213.917$\std{19.863}  & $389.0e3$\std{0.5e3} & $6.33$\\
GP-qEI          & $0.824$\std{0.086}     & $0.675$\std{0.043}     & $92.686$\std{3.944}       & $480.049$\std{0.000}  & $213.816$\std{0.000}  & $388.1e3$\std{0.0} & $7.17$\\
CMA-ES          & $0.933$\std{0.035}     & $0.848$\std{0.136}     & $90.821$\std{0.661}         & $\mathbf{1016.409}$\std{906.407} & $4.700$\std{2.230}   & $388.4e3$\std{0.4e3} & $5.17$\\
Gradient Ascent & $\underline{0.981}$\std{0.010}     & $0.770$\std{0.154}      & $93.252$\std{0.886}        & $-54.955$\std{33.482}  & $226.491$\std{21.120}  & $390.1e3$\std{2.0e3} & $4.33$\\
REINFORCE            & $0.959$\std{0.013}     & $0.692$\std{0.113}      & $89.469$\std{3.093}         & $-131.907$\std{41.003} & $-301.866$\std{246.284} & $388.4e3$\std{2.1e3} & $7.33$\\
MINs       & $0.938$\std{0.047}     & $0.770$\std{0.177}      & $89.027$\std{3.093}         & $533.636$\std{17.938}  & $\underline{272.675}$\std{11.069}  & $\underline{391.0e3}$\std{0.2e3} & $4.50$\\
COMs            & $0.964$\std{0.020}     & $0.750$\std{0.078}      & $78.178$\std{6.179}         & $540.603$\std{20.205}  & $\mathbf{277.888}$\std{7.799}  & $390.2e3$\std{0.5e3}  & $4.50$\\
DDOM            & $0.971$\std{0.005}     & $\underline{0.885}$\std{0.367}      & $\underline{103.600}$\std{8.139}        & $\underline{548.227}$\std{11.725}  & $250.529$\std{10.992}  & $388.0e3$\std{1.1e3} \vspace{5pt} & $\underline{3.67}$\\
\hline \addlinespace[5pt]
Ours      & $\mathbf{0.987}$ \std{0.014}    &   $\textbf{0.924}$\std{0.224}         & $\mathbf{113.545}$\std{5.322}      &  $493.191$\std{18.165} 
      &  $261.673$\std{3.486}        &$\mathbf{391.1e3}$\std{3.4e3}  &$\mathbf{2.00}$\\
\addlinespace[2pt]
\bottomrule[1.5pt]
\end{tabular}
}
\vspace{-1em}
\caption{Results of offline black-box optimization on DesignBench. We report the mean and standard deviation across five random seeds. The best results are \textbf{bolded}, and the second best is \underline{underlined}.}
%\vspace{-0.4em}
\label{tab:designbench}
\end{table*}

\begin{table*}[]
\small
\centering
\resizebox{\textwidth}{!}{
\begin{tabular}{c|cccc|cccc|c}
\bottomrule[1.5pt]
& \multicolumn{4}{c|}{Top-1} & \multicolumn{4}{c|}{Top-10} &  \multirow{2}{*}{Invalidity $\downarrow$ }\\\cline{2-9}
& QED $\uparrow$   & SA $\downarrow$   & GSK3B $\uparrow$  & Sum $\uparrow$ & QED $\uparrow$   & SA $\downarrow$   & GSK3B $\uparrow$ & Sum $\uparrow$ &   \\  \hline
Dataset Mean       & $0.598$ & $0.204$ & $0.045$ & $0.439$ & $0.598$ & $0.204$ & $0.045$ & $0.439$ &  $0$ \\
Dataset Sum Best  & $0.846$ & $0.159$ & $0.99$  & $1.677$ & $0.771$ & $0.129$ & $0.877$ & $1.519$ &  $0$ \\
Dataset Best        & $0.947$ & $0.030$  & $0.99$  & $1.907$ & $0.945$ & $0.204$ & $0.947$ & $1.688$ &  $0$ \\ \hline
DDOM            & $0.790$\std{0.023} & $0.124$\std{0.007} & $\underline{0.856}$\std{0.046}  & $\underline{1.521}$\std{0.063} & $0.747$\std{0.033} & $0.141$\std{0.006} & $\underline{0.695}$\std{0.021} & $\underline{1.301}$\std{0.037} & $63.60$\std{3.61} \\
Gradient Ascent & $\underline{0.834}$\std{0.025} & $0.130$\std{0.024} & $0.784$\std{0.152}  & $1.487$\std{0.150} & $0.674$\std{0.043} & $0.134$\std{0.010} & $0.678$\std{0.047} & $1.218$\std{0.091} & $80.00$\std{19.20}        \\
GP-qEI          & $0.784$\std{0.059} & $0.149$\std{0.034} & $0.551$\std{0.106}  & $1.186$\std{0.089} & $0.743$\std{0.032} & $0.147$\std{0.010} & $0.370$\std{0.057} & $0.966$\std{0.028} & $63.80$\std{7.62}         \\
MINs            & $0.795$\std{0.156} & $0.163$\std{0.027} & $0.466$\std{0.249}  & $1.097$\std{0.200} & $\mathbf{0.838}$\std{0.059} & $0.145$\std{0.029} & $0.273$\std{0.149} & $0.966$\std{0.108} & $\underline{38.20}$\std{24.66}        \\
REINFORCE       & $\mathbf{0.865}$\std{0.047} & $\mathbf{0.083}$\std{0.013} & $0.062$\std{0.069}  & $0.843$\std{0.027} & $\underline{0.816}$\std{0.056} & $\mathbf{0.079}$\std{0.009} & $0.085$\std{0.072} & $0.822$\std{0.026} & $53.40$\std{106.80}       \\
CbAS            & $0.762$\std{0.119} & $0.138$\std{0.046} & $0.681$\std{0.077}  & $1.305$\std{0.049} & $0.687$\std{0.035} & $0.153$\std{0.010} & $0.593$\std{0.078} & $1.126$\std{0.059} & $46.80$\std{13.90}        \\
CMA-ES          & $0.446$\std{0.011} & $0.207$\std{0.122} & $0.012$\std{0.004}  & $0.250$\std{0.128} & $0.435$\std{0.033} & $0.230$\std{0.169} & $0.008$\std{0.001}  & $0.212$\std{0.205} & $880.00$\std{59.25} \\
\hline
Ours & $0.798$\std{0.023} & $\underline{0.100}$\std{0.031} & $\mathbf{0.944}$\std{0.023} & $\mathbf{1.641}$\std{0.018} & $0.786$\std{0.003} & $\underline{0.100}$\std{0.006} & $\mathbf{0.866}$\std{0.035} & $\mathbf{1.552}$\std{0.031} & $\mathbf{35.80}$\std{1.46} \\
\bottomrule
\end{tabular}
}
\vspace{-1em}
\caption{Results on the multi-objective molecule optimization task. Sum denotes the total objective (QED $+$ GSK3B $-$ SA). SA is normalized from the range 1-10 to 0-1. We report the mean and standard deviation across five random seeds.
The best results are \textbf{bolded}, and the second best is \underline{underlined}. Top-1 denotes the best solution found, and Top-10 denotes the average of the best ten solutions found. Invalidity denotes the number of invalid molecules in the 1000 generated samples.  }
\label{tab:multi}
%\vspace{-0.8em}
\end{table*}

In this section, we conduct experiments on (1) a synthetic Branin task, (2) six real-world offline black-box optimization tasks, and (3) a multi-objective molecule optimization task. Finally, we do ablation studies.

\subsection{Synthetic Branin Function Optimization}
We first validate our model on a synthetic Branin task. The Branin function \citep{dixon1978global} is given as
\begin{equation}
    f(x_1,x_2)=a(x_2-bx_1^2+cx_1-r)^2+s(1-t)\text{cos}(x_1)+s , 
\end{equation}
where $a=1$, $b=\frac{5.1}{4\pi^2}$, $c=\frac{5}{\pi}$, $r=6$, $s=10$, $t=\frac{1}{8\pi}$. The function has three  global minimas, $(-\pi,12.275), (\pi, 2.275)$, and $(9.42478, 2.475)$. %(\cref{fig:branin}). %with the minimum value of $-0.398$ (\cref{fig:branin}).

\textbf{Optimization with unknown constraints.} 
% \begin{wrapfigure}{r}{0.25\textwidth}
%   \begin{center}  \includegraphics[width=0.25\textwidth]{plot/branin-positive.pdf}
%   \end{center}
%   %\vspace{-1.5em}
%   \caption{Branin function.}
%   \label{fig:branin}
%   %\vspace{-1.2em}
% \end{wrapfigure}
To assess the capability of \ours~ in optimizing functions under unknown constraints, we generate a dataset of 6,000 points, uniformly distributed within the feasible domain shaped like an oval. An effective optimizer is expected to infer the feasible space from the dataset and yield solutions strictly within the permissible region, \ie~$(-\pi,12.275)$ and $(\pi, 2.275)$. We train a diffusion model with Variance Preserving (VP) SDE \citep{song2020score} on this dataset. More details of the experimental setup are provided in  
% supplementary materials.
Appendix~\ref{appendix:exp}.
Figure~\ref{fig:landscape_unknown} illustrates the sampling trajectory of \ours, clearly demonstrating its capability in guiding the samples towards the optimal points confined to the feasible space.

\textbf{Compatible with additional known constraints.} 
Our framework is also adaptable to scenarios with additional, known constraints $C'$. In such instances, we introduce an extra objective function whose Boltzmann density is uniform within the constraint bounds and $0$ otherwise, \ie~$q'(x) \propto \exp[\beta' \cdot I(x \in C')]$, with $I(\cdot)$ being the indicator function. To demonstrate this capability, we incorporate a closed-form linear constraint alongside the implicit constraint represented by the dataset. This new constraint narrows the feasible solutions to only $(\pi, 2.275)$. As depicted in \cref{fig:landscape_partial_unknown}, \ours effectively navigates towards the sole viable minimizer within the constrained space delineated by the data-driven and explicitly stated constraints. This feature is particularly beneficial in practical applications, such as molecular optimization, where imposing additional spatial or structure constraints might be necessary during optimizing binding affinities with different protein targets \citep{du2022molgensurvey}.

%\vspace{-1em}
\subsection{Offline Black-box Optimization}
%\vspace{-1em}
We further evaluate \ours on the offline black-box optimization task, wherein a logged dataset is utilized to train a surrogate model that approximates the objective function. The surrogate model is trained on finite data and may have large fitness errors beyond the data distribution. At inference time, if we directly apply gradient ascent on the surrogate objective, it may produce out-of-distribution designs that “fool” the learned surrogate model into outputting a high value. Therefore, we need to constrain the optimization process within the data distribution because the trained surrogate is only reliable within this range. However, this data distribution constraints cannot be expressed analytically, and therefore this task can be viewed as an instance of optimization with unknown constraints.

Following \citep{DDOM}, we conduct evaluation on six tasks of DesignBench \citep{designbench}. \textbf{Superconductor} is to optimize for finding a superconducting material with high critical temperature. \textbf{Ant} nad \textbf{D'Kitty} is to optimize the robot morphology. \textbf{TFBind8} and \textbf{TFBind10} are to find a DNA sequence that maximizes binding affinity with specific transcription factors. \textbf{ChEMBL} is to optimize the drugs for a particular chemical property.

\textbf{Baselines.}
We compare \ours with multiple baselines, including gradient ascent, Bayesian optimization (GP-qEI) \citep{DDOM},  REINFORCE \citep{REINFORCE}, evolutionary algorithm (CAM-ES) \citep{Hansen}, and recent methods like MINS \citep{KumanMINs}, COMs \citep{COMs}, CbAS \citep{CbAS} and DDOM \citep{DDOM}. We follow \citep{DDOM} and set the sampling budget as $256$. More details of the experimental setups are provided in Appendix~\ref{appendix:exp}.

\textbf{Results.} Table~\ref{tab:designbench} shows the performance on the six datasets for all the methods. As we can see, \ours achieves an average rank of $2.0$, the best among all the methods. We achieve the best result on 4 tasks. Particularly, in the Superconductor task, \ours surpasses all baseline methods by a significant margin, improving upon the closest competitor by 9.6\%. The exceptional performance of \ours is primarily due to its application of a diffusion model to learn the valid data manifold directly from the data set, thus rendering the optimization process significantly more reliable. In contrast, the 
gradient ascent method, which relies solely on optimizing the trained surrogate model, is prone to settle on suboptimal solutions. 
Moreover, while DDOM \citep{DDOM} employs a conditional diffusion model to learn an inverse mapping from objective values to the input space, its ability to generate samples is confined to the maximum values present in the offline dataset. This limitation restricts its ability to identify global maximizers within the feasible space. The experimental results also demonstrate that \ours can consistently outperform DDOM except for the Ant Dataset. %On Ant, we find that the suboptimal performance of \ours is because the surrogate model is extremely difficult to train. We provide more analyses in Appendix~\ref{section:ant}.

\vspace{-0.5em}
\subsection{Multi-objective Molecule Optimization}
\vspace{-0.5em}
An additional advantage of incorporating the data distribution constraints for offline black-box optimization is their direct impact on enhancing the validity of generated solutions. However, it's worth noting that DesignBench lacks a specific metric for assessing validity. Therefore, we further test on a multi-objective molecule optimization task and extend our evaluation to include validity. In this task, we have three objectives: the maximization of the quantitative estimate of drug-likeness (QED), and the activity against glycogen synthase kinase three beta enzyme (GSK3B), and the minimization of the synthetic accessibility score (SA). Following \citep{hiervae,limo}, we utilize a pre-trained autoencoder from \citep{hiervae} to map discrete molecular structures into a continuous low-dimensional latent space and train neural networks as proxy functions for predicting molecular properties to guide the optimization. Detailed experimental setups are provided in the 
% supplementary material.
Appendix~\ref{appendix:exp}.

\textbf{Results.}
For each method, we generate 1,000 candidate solutions, evaluating the three objective metrics solely on those that are valid. As we can see from \cref{tab:multi}, \ours can achieve the best validity performance among all the methods. In terms of optimization performance, \ours can achieve the best overall objective value. We further report the average of the top 10 solutions found by each model and find that DiffOPT is also reliable in this scenario.
This multi-objective optimization setting is particularly challenging, as different objectives can conflict with each other. The superior performance of \ours is because we formulate the optimization problem as a sample problem from the product of experts, which is easy for compositions of various objectives.

\begin{figure*}[t]
    \centering \includegraphics[width=0.95\textwidth]{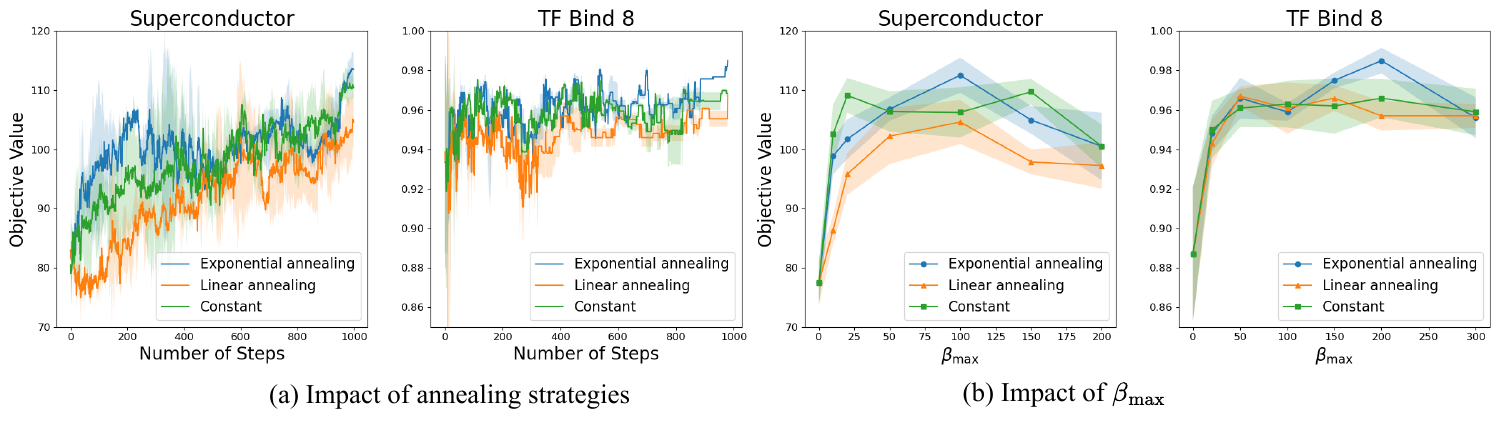}
    \vspace{-1.5em}
    \caption{Impact of annealing strategies and $\beta_{\rm max}$ in the guided diffusion stage. $\beta_{\rm max}$ is the value of $\beta$ at the end of annealing.}
    \label{fig:annealing-main}
\end{figure*}

\subsection{Derivative-free Optimization}

\begin{table}[h]
% \begin{wraptable}{r}{0.4\textwidth} % 'r' for right alignment, 0.5\textwidth for width
    \centering
    \begin{tabular}{@{}lll@{}}
        \toprule
             & Hypervolume $\uparrow$ & Invalidity $\downarrow$ \\ \midrule
        EA   & $0.558$\std{0.029}  & $37.6$\std{8.73}      \\
        IS & $0.493$\std{0.108} & $\mathbf{0.0}$\std{0.0} \\
        Ours & $\mathbf{0.590}$\std{0.060} & $\mathbf{0.0}$\std{0.0} \\ \bottomrule
    \end{tabular}
    \caption{Results of derivative-free optimization on the multi-objective molecule optimization task.}
    \label{table:derivative-free}
% \end{wraptable}
\end{table}

In this subsection, we explore the scenario where the objective function is not differentiable and examine the effectiveness of the proposed iterative importance sampling with diffusion models. We continue using the molecule optimization task. Here, we do not train a surrogate objective using offline data but instead directly use the oracle function from Rdkit \citep{rdkit}. We compare our method with evolutionary algorithm (EA; \citep{Hansen}), which is a common technique in derivative-free multi-objective optimization. Additionally, we compare our method with one-step importance sampling (IS) using diffusion models as the proposal distribution. \Cref{table:derivative-free} provides the performance of all the methods. As we can see, \ours can achieve the best hypervolume and validity among all the methods. 

\subsection{Ablation Study}

\begin{table}[h]
\centering
% \begin{wraptable}{r}{0.55\textwidth}
\small
\resizebox{0.45\textwidth}{!}{

\begin{tabular}{@{}lcc@{}}
\toprule
                        & Superconductor & TFBind8 \\ \midrule
Best Baseline           &   $103.600\pm 8.139$             &    $0.981\pm 0.010$       \\
Only Stage I            &     $112.038\pm 6.783$           &  $0.984\pm 0.012 $         \\
Only Stage II           &    $92.432\pm 8.635$            &   $0.951\pm 0.028$        \\
Stage I + Stage II      &     $113.545\pm 5.322$           &         $0.987\pm 0.014$  \\
Stage I + Stage II + MH &         $\textbf{114.945}\pm 3.615$       &        $\textbf{0.989}\pm 0.021$   \\ \bottomrule
\end{tabular}}
\caption{Ablation study on the two-stage sampling.}
\label{table:ablation-two-stage-main}

% \end{wraptable}
\end{table}

\textbf{Impact of two-stage sampling.}
\cref{table:ablation-two-stage-main}  shows the impact of two-stage sampling on performance. Our findings reveal that even after the initial stage, \ours outperforms the top-performing baseline on both datasets. Relying solely on Langevin dynamics, without the warm-up phase of guided diffusion, results in significantly poorer results. This aligns with our discussion in \cref{sec:two-stage}, where we attributed this failure to factors such as the starting distribution, the schedule for step size adjustments, and the challenges posed by undefined gradients outside the feasible set. Integrating both stages yields a performance improvement as the initial stage can provide a better initialization within the data manifold  for the later stage (\cref{thm:concentration_warmup}). Adding the MH correction step further enhances results, leading to the best performance observed.

\textbf{Impact of annealing strategies.}  We study the influence of different annealing strategies for $\beta$ during the guided diffusion stage, focusing on the superconductor and TFBind8 datasets. We explore three strategies: constant, linear annealing, and exponential annealing. Figure~\ref{fig:annealing-main}(a) presents the performance across various diffusion steps. We find that our method is not particularly sensitive to the annealing strategies. However, it is worth noting that exponential annealing exhibits a marginal performance advantage over the others.

We also investigate how the value of $\beta$ at the end of annealing, denoted as $\beta_{\rm max}$, affects model performance in Figure~\ref{fig:annealing-main}(b). We find that increasing $\beta_{\rm max}$ initially leads to better performance. However, beyond a certain threshold, performance begins to decrease. It is noteworthy that the optimal value varies across different annealing strategies. Particularly, at $\beta_{\rm max} = 0$, the model reverts to a pure diffusion process, exhibiting the lowest performance due to the lack of guidance from the objective function.

\textbf{Sample efficiency.} We explore the sample efficiency of \ours at both training and testing stages. \Cref{fig:training_data_size} (in \Cref{sec:appendix:ablation}) shows the performance of various methods versus the ratio of training data on Superconductor, TFBind8 and multi-objective molecule optimization. As we can see, on all the three datasets, DiffOPT can outperform all. Our method is also sample efficient during inference. \Cref{fig:training_sample} (in \Cref{sec:appendix:ablation}) shows the performance versus number of samples at inference stage. Notably, on both Superconductor and TFBind8, DiffOPT consistently outperforms all the baseline methods for various sample sizes during inference. It is also important to highlight that our method consistently achieves much greater sample efficiency than DDOM at both training and inference stages, despite both approaches leveraging diffusion models.

% \cref{table:ablation-two-stage} (in Appendix~\ref{sec:appendix:ablation}) shows the impact of two-stage sampling on performance. Our findings reveal that even after the initial stage, \ours outperforms the top-performing baseline on both datasets. Relying solely on Langevin dynamics, without the warm-up phase of guided diffusion, results in significantly poorer results. This aligns with our discussion in \cref{sec:two-stage}, where we attributed this failure to factors such as the starting distribution, the schedule for step size adjustments, and the challenges posed by undefined gradients outside the feasible set. Integrating both stages yields a performance improvement as the initial stage can provide a better initialization within the data manifold  for the later stage (\cref{thm:concentration_warmup}). Adding the MH correction step further enhances results, leading to the best performance observed. 
% We also study the impact of annealing strategies and sample efficiency of \ours in 
% % supplementary materials.
% \cref{sec:appendix:ablation}.

\vspace{-0.5em}

\section{Conclusion}
\vspace{-0.5em}
In this paper, we propose \ours to solve optimization problems where analytic constraints are unavailable. We learn the unknown feasible space from data using a diffusion model and then reformulate the original problem as a sampling problem from the product of (i) the density of the data distribution learned by the diffusion model and (ii) the Boltzmann density defined by the objective function. For differentiable objectives, we propose a two-stage framework consisting of a guided diffusion stage for warm-up and a Langevin dynamics stage for further correction. For non-differentiable objectives, we propose an iterative importance sampling method with diffusion models. Our experiments validate the effectiveness of \ours. %Due to the space limit, we discuss limitations and future work in 

\section*{Acknowledgement}

This work was supported in part by NSF IIS-2008334, IIS-2106961, IIS-2403240, and CAREER IIS-2144338. Yuanqi Du and Carla P. Gomes was supported by the Eric and Wendy Schmidt AI in Science Postdoctoral Fellowship, a Schmidt Futures program; the National Science Foundation (NSF), the Air
Force Office of Scientific Research (AFOSR); the Department of Energy; and the Toyota Research Institute (TRI). Dongxia Wu and Yi-an Ma was supported by the U.S. Department Of Energy, Office of Science, and NSF Grant 2112665 (TILOS), CDC-RFA-FT-23-0069 and DARPA AIE FoundSci.

Y.D. would like to thank Yunan Yang, Guan-horng Liu, Ricky T.Q. Chen, Ge Liu and Yimeng Zeng for their helpful discussions.

\bibliographystyle{plainnat}
\bibliography{references}

\begin{thebibliography}{122}
\providecommand{\natexlab}[1]{#1}
\providecommand{\url}[1]{\texttt{#1}}
\expandafter\ifx\csname urlstyle\endcsname\relax
  \providecommand{\doi}[1]{doi: #1}\else
  \providecommand{\doi}{doi: \begingroup \urlstyle{rm}\Url}\fi

\bibitem[Alarie et~al.(2021)Alarie, Audet, Gheribi, Kokkolaras, and Le~Digabel]{alarie2021two}
St{\'e}phane Alarie, Charles Audet, A{\"\i}men~E Gheribi, Michael Kokkolaras, and S{\'e}bastien Le~Digabel.
\newblock Two decades of blackbox optimization applications.
\newblock \emph{EURO Journal on Computational Optimization}, 9:\penalty0 100011, 2021.

\bibitem[Anderson(1982)]{anderson1982reverse}
Brian~DO Anderson.
\newblock Reverse-time diffusion equation models.
\newblock \emph{Stochastic Processes and their Applications}, 12\penalty0 (3):\penalty0 313--326, 1982.

\bibitem[Audet and Dennis~Jr(2004)]{audet2004pattern}
Charles Audet and John~E Dennis~Jr.
\newblock A pattern search filter method for nonlinear programming without derivatives.
\newblock \emph{SIAM Journal on Optimization}, 14\penalty0 (4):\penalty0 980--1010, 2004.

\bibitem[Audet and Dennis~Jr(2006)]{audet2006mesh}
Charles Audet and John~E Dennis~Jr.
\newblock Mesh adaptive direct search algorithms for constrained optimization.
\newblock \emph{SIAM Journal on optimization}, 17\penalty0 (1):\penalty0 188--217, 2006.

\bibitem[Baladat et~al.(2020)Baladat, Karrer, Jiang, Daulton, Letham, Wilson, and Bakshy]{botorch}
Maximilian Baladat, Brian Karrer, Daniel~R. Jiang, Samuel Daulton, Benjamin Letham, Andrew~Gordon Wilson, and Eytan Bakshy.
\newblock Botorch: Programmable bayesian optimization in pytorch.
\newblock \emph{Advances in Neural Information Processing Systems}, 2020.

\bibitem[Baran(2018)]{baran2018natural}
Phil~S Baran.
\newblock Natural product total synthesis: as exciting as ever and here to stay, 2018.

\bibitem[Bazaraa et~al.(2011)Bazaraa, Jarvis, and Sherali]{bazaraa2011linear}
Mokhtar~S Bazaraa, John~J Jarvis, and Hanif~D Sherali.
\newblock \emph{Linear programming and network flows}.
\newblock John Wiley \& Sons, 2011.

\bibitem[Benton et~al.(2023)Benton, De~Bortoli, Doucet, and Deligiannidis]{benton2023linear}
Joe Benton, Valentin De~Bortoli, Arnaud Doucet, and George Deligiannidis.
\newblock Linear convergence bounds for diffusion models via stochastic localization.
\newblock \emph{arXiv preprint arXiv:2308.03686}, 2023.

\bibitem[Boyd and Vandenberghe(2004)]{boyd2004convex}
Stephen~P Boyd and Lieven Vandenberghe.
\newblock \emph{Convex optimization}.
\newblock Cambridge university press, 2004.

\bibitem[Branke and Schmidt(2004)]{branke2004sequential}
J{\"u}rgen Branke and Christian Schmidt.
\newblock Sequential sampling in noisy environments.
\newblock In \emph{International Conference on Parallel Problem Solving from Nature}, pages 202--211. Springer, 2004.

\bibitem[Brookes et~al.(2019)Brookes, Park, and Listgarten]{CbAS}
David~H Brookes, Hahnbeom Park, and Jennifer Listgarten.
\newblock Conditioning by adaptive sampling for robust design.
\newblock \emph{International Conference on Machine Learning}, 2019.

\bibitem[Burke et~al.(2005)Burke, Lewis, and Overton]{burke2005robust}
James~V Burke, Adrian~S Lewis, and Michael~L Overton.
\newblock A robust gradient sampling algorithm for nonsmooth, nonconvex optimization.
\newblock \emph{SIAM Journal on Optimization}, 15\penalty0 (3):\penalty0 751--779, 2005.

\bibitem[Burke et~al.(2020)Burke, Curtis, Lewis, Overton, and Sim{\~o}es]{burke2020gradient}
James~V Burke, Frank~E Curtis, Adrian~S Lewis, Michael~L Overton, and Lucas~EA Sim{\~o}es.
\newblock Gradient sampling methods for nonsmooth optimization.
\newblock \emph{Numerical nonsmooth optimization: State of the art algorithms}, pages 201--225, 2020.

\bibitem[Cattiaux et~al.(2023)Cattiaux, Conforti, Gentil, and L{\'e}onard]{cattiaux2023time}
Patrick Cattiaux, Giovanni Conforti, Ivan Gentil, and Christian L{\'e}onard.
\newblock Time reversal of diffusion processes under a finite entropy condition.
\newblock In \emph{Annales de l'Institut Henri Poincar{\'e} (B) Probabilit{\'e}s et Statistiques}, volume~59, pages 1844--1881. Institut Henri Poincar{\'e}, 2023.

\bibitem[Chen et~al.(2023)Chen, Zhang, Liu, and Coates]{chen2023bidirectional}
Can Chen, Yingxue Zhang, Xue Liu, and Mark Coates.
\newblock Bidirectional learning for offline model-based biological sequence design.
\newblock \emph{arXiv preprint arXiv:2301.02931}, 2023.

\bibitem[Chen et~al.(2022)Chen, Chewi, Li, Li, Salim, and Zhang]{chen2022sampling}
Sitan Chen, Sinho Chewi, Jerry Li, Yuanzhi Li, Adil Salim, and Anru~R Zhang.
\newblock Sampling is as easy as learning the score: theory for diffusion models with minimal data assumptions.
\newblock \emph{arXiv preprint arXiv:2209.11215}, 2022.

\bibitem[Cheng(2020)]{cheng2020interplay}
Xiang Cheng.
\newblock \emph{The Interplay between Sampling and Optimization}.
\newblock University of California, Berkeley, 2020.

\bibitem[Cheng et~al.(2020)Cheng, Yin, Bartlett, and Jordan]{cheng2020stochastic}
Xiang Cheng, Dong Yin, Peter Bartlett, and Michael Jordan.
\newblock Stochastic gradient and langevin processes.
\newblock In \emph{International Conference on Machine Learning}, pages 1810--1819. PMLR, 2020.

\bibitem[Chung et~al.(2022{\natexlab{a}})Chung, Kim, Mccann, Klasky, and Ye]{chung2022diffusion}
Hyungjin Chung, Jeongsol Kim, Michael~T Mccann, Marc~L Klasky, and Jong~Chul Ye.
\newblock Diffusion posterior sampling for general noisy inverse problems.
\newblock \emph{arXiv preprint arXiv:2209.14687}, 2022{\natexlab{a}}.

\bibitem[Chung et~al.(2022{\natexlab{b}})Chung, Sim, Ryu, and Ye]{chung2022improving}
Hyungjin Chung, Byeongsu Sim, Dohoon Ryu, and Jong~Chul Ye.
\newblock Improving diffusion models for inverse problems using manifold constraints.
\newblock \emph{Advances in Neural Information Processing Systems}, 35:\penalty0 25683--25696, 2022{\natexlab{b}}.

\bibitem[Conforti et~al.(2023)Conforti, Durmus, and Silveri]{conforti2023score}
Giovanni Conforti, Alain Durmus, and Marta~Gentiloni Silveri.
\newblock Score diffusion models without early stopping: finite fisher information is all you need.
\newblock \emph{arXiv preprint arXiv:2308.12240}, 2023.

\bibitem[Conn et~al.(2009)Conn, Scheinberg, and Vicente]{conn2009introduction}
Andrew~R Conn, Katya Scheinberg, and Luis~N Vicente.
\newblock \emph{Introduction to derivative-free optimization}.
\newblock SIAM, 2009.

\bibitem[De~Bortoli(2022)]{debortoli2022convergence}
Valentin De~Bortoli.
\newblock Convergence of denoising diffusion models under the manifold hypothesis.
\newblock \emph{arXiv preprint arXiv:2208.05314}, 2022.

\bibitem[De~Bortoli and Desolneux(2021)]{debortoli2021quantitative}
Valentin De~Bortoli and Agn{\`e}s Desolneux.
\newblock On quantitative laplace-type convergence results for some exponential probability measures, with two applications.
\newblock \emph{arXiv preprint arXiv:2110.12922}, 2021.

\bibitem[De~Bortoli et~al.(2021)De~Bortoli, Thornton, Heng, and Doucet]{de2021diffusion}
Valentin De~Bortoli, James Thornton, Jeremy Heng, and Arnaud Doucet.
\newblock Diffusion schr{\"o}dinger bridge with applications to score-based generative modeling.
\newblock \emph{Advances in Neural Information Processing Systems}, 34:\penalty0 17695--17709, 2021.

\bibitem[Dixon(1978)]{dixon1978global}
Laurence Charles~Ward Dixon.
\newblock The global optimization problem: an introduction.
\newblock \emph{Towards Global Optimiation 2}, pages 1--15, 1978.

\bibitem[Du et~al.(2023{\natexlab{a}})Du, Zhang, Yang, and Du]{du2023flexible}
Weitao Du, He~Zhang, Tao Yang, and Yuanqi Du.
\newblock A flexible diffusion model.
\newblock In \emph{International Conference on Machine Learning}, pages 8678--8696. PMLR, 2023{\natexlab{a}}.

\bibitem[Du et~al.(2023{\natexlab{b}})Du, Durkan, Strudel, Tenenbaum, Dieleman, Fergus, Sohl-Dickstein, Doucet, and Grathwohl]{du2023reduce}
Yilun Du, Conor Durkan, Robin Strudel, Joshua~B Tenenbaum, Sander Dieleman, Rob Fergus, Jascha Sohl-Dickstein, Arnaud Doucet, and Will~Sussman Grathwohl.
\newblock Reduce, reuse, recycle: Compositional generation with energy-based diffusion models and mcmc.
\newblock In \emph{International Conference on Machine Learning}, pages 8489--8510. PMLR, 2023{\natexlab{b}}.

\bibitem[Du et~al.(2022)Du, Fu, Sun, and Liu]{du2022molgensurvey}
Yuanqi Du, Tianfan Fu, Jimeng Sun, and Shengchao Liu.
\newblock Molgensurvey: A systematic survey in machine learning models for molecule design.
\newblock \emph{arXiv preprint arXiv:2203.14500}, 2022.

\bibitem[Durmus and Moulines(2022)]{durmus2022geometric}
Alain Durmus and {\'E}ric Moulines.
\newblock On the geometric convergence for mala under verifiable conditions.
\newblock \emph{arXiv preprint arXiv:2201.01951}, 2022.

\bibitem[Eckmann et~al.(2022)Eckmann, Sun, Zhao, Feng, Gilson, and Yu]{limo}
Peter Eckmann, Kunyang Sun, Bo~Zhao, Mudong Feng, Michael Gilson, and Rose Yu.
\newblock Limo: Latent inceptionism for targeted molecule generation.
\newblock In \emph{International Conference on Machine Learning}, pages 5777--5792. PMLR, 2022.

\bibitem[Fasano et~al.(2014)Fasano, Liuzzi, Lucidi, and Rinaldi]{fasano2014linesearch}
Giovanni Fasano, Giampaolo Liuzzi, Stefano Lucidi, and Francesco Rinaldi.
\newblock A linesearch-based derivative-free approach for nonsmooth constrained optimization.
\newblock \emph{SIAM journal on optimization}, 24\penalty0 (3):\penalty0 959--992, 2014.

\bibitem[Fu and Levine(2020)]{fu2020offline}
Justin Fu and Sergey Levine.
\newblock Offline model-based optimization via normalized maximum likelihood estimation.
\newblock In \emph{International Conference on Learning Representations}, 2020.

\bibitem[Gao and Coley(2020)]{gao2020synthesizability}
Wenhao Gao and Connor~W Coley.
\newblock The synthesizability of molecules proposed by generative models.
\newblock \emph{Journal of chemical information and modeling}, 60\penalty0 (12):\penalty0 5714--5723, 2020.

\bibitem[Garipov et~al.(2023)Garipov, Peuter, Yang, Garg, Kaski, and Jaakkola]{garipov2023compositional}
Timur Garipov, Sebastiaan~De Peuter, Ge~Yang, Vikas Garg, Samuel Kaski, and Tommi~S. Jaakkola.
\newblock Compositional sculpting of iterative generative processes.
\newblock In \emph{Thirty-seventh Conference on Neural Information Processing Systems}, 2023.
\newblock URL \url{https://openreview.net/forum?id=w79RtqIyoM}.

\bibitem[Gasse et~al.(2019)Gasse, Ch{\'e}telat, Ferroni, Charlin, and Lodi]{gasse2019exact}
Maxime Gasse, Didier Ch{\'e}telat, Nicola Ferroni, Laurent Charlin, and Andrea Lodi.
\newblock Exact combinatorial optimization with graph convolutional neural networks.
\newblock \emph{Advances in neural information processing systems}, 32, 2019.

\bibitem[Gaulton et~al.(2012)Gaulton, Bellis, Bento, Chambers, Davies, Hersey, Light, McGlinchey, Michalovich, Al-Lazikani, et~al.]{gaulton2012chembl}
Anna Gaulton, Louisa~J Bellis, A~Patricia Bento, Jon Chambers, Mark Davies, Anne Hersey, Yvonne Light, Shaun McGlinchey, David Michalovich, Bissan Al-Lazikani, et~al.
\newblock Chembl: a large-scale bioactivity database for drug discovery.
\newblock \emph{Nucleic acids research}, 40\penalty0 (D1):\penalty0 D1100--D1107, 2012.

\bibitem[Gelfand and Mitter(1991)]{gelfand1991recursive}
Saul~B Gelfand and Sanjoy~K Mitter.
\newblock Recursive stochastic algorithms for global optimization in r\^{}d.
\newblock \emph{SIAM Journal on Control and Optimization}, 29\penalty0 (5):\penalty0 999--1018, 1991.

\bibitem[Goodfellow et~al.(2014)Goodfellow, Pouget-Abadie, Mirza, Xu, Warde-Farley, Ozair, Courville, and Bengio]{goodfellow2014generative}
Ian Goodfellow, Jean Pouget-Abadie, Mehdi Mirza, Bing Xu, David Warde-Farley, Sherjil Ozair, Aaron Courville, and Yoshua Bengio.
\newblock Generative adversarial nets.
\newblock \emph{Advances in neural information processing systems}, 27, 2014.

\bibitem[Grenander and Miller(1994)]{grenander1994representations}
Ulf Grenander and Michael~I Miller.
\newblock Representations of knowledge in complex systems.
\newblock \emph{Journal of the Royal Statistical Society: Series B (Methodological)}, 56\penalty0 (4):\penalty0 549--581, 1994.

\bibitem[Gruver et~al.(2023)Gruver, Stanton, Frey, Rudner, Hotzel, Lafrance-Vanasse, Rajpal, Cho, and Wilson]{gruver2023protein}
Nate Gruver, Samuel~Don Stanton, Nathan~C. Frey, Tim G.~J. Rudner, Isidro Hotzel, Julien Lafrance-Vanasse, Arvind Rajpal, Kyunghyun Cho, and Andrew~Gordon Wilson.
\newblock Protein design with guided discrete diffusion.
\newblock In \emph{Thirty-seventh Conference on Neural Information Processing Systems}, 2023.
\newblock URL \url{https://openreview.net/forum?id=MfiK69Ga6p}.

\bibitem[Hansen()]{Hansen}
Nikolaus Hansen.
\newblock The {CMA} evolution strategy: A comparing review.
\newblock In \emph{Towards a New Evolutionary Computation}, pages 75--102.

\bibitem[Hastings(1970)]{hastings1970monte}
W~Keith Hastings.
\newblock Monte carlo sampling methods using markov chains and their applications.
\newblock 1970.

\bibitem[Haussmann and Pardoux(1986)]{haussmann1986time}
Ulrich~G Haussmann and Etienne Pardoux.
\newblock Time reversal of diffusions.
\newblock \emph{The Annals of Probability}, pages 1188--1205, 1986.

\bibitem[He et~al.(2023)He, Murata, Lai, Takida, Uesaka, Kim, Liao, Mitsufuji, Kolter, Salakhutdinov, et~al.]{he2023manifold}
Yutong He, Naoki Murata, Chieh-Hsin Lai, Yuhta Takida, Toshimitsu Uesaka, Dongjun Kim, Wei-Hsiang Liao, Yuki Mitsufuji, J~Zico Kolter, Ruslan Salakhutdinov, et~al.
\newblock Manifold preserving guided diffusion.
\newblock \emph{arXiv preprint arXiv:2311.16424}, 2023.

\bibitem[Hinton(2002)]{hinton2002training}
Geoffrey~E Hinton.
\newblock Training products of experts by minimizing contrastive divergence.
\newblock \emph{Neural computation}, 14\penalty0 (8):\penalty0 1771--1800, 2002.

\bibitem[Ho et~al.(2020)Ho, Jain, and Abbeel]{ddpm}
Jonathan Ho, Ajay Jain, and Pieter Abbeel.
\newblock Denoising diffusion probabilistic models.
\newblock \emph{Advances in neural information processing systems}, 33:\penalty0 6840--6851, 2020.

\bibitem[Ho et~al.(2022)Ho, Chan, Saharia, Whang, Gao, Gritsenko, Kingma, Poole, Norouzi, Fleet, et~al.]{ho2022imagen}
Jonathan Ho, William Chan, Chitwan Saharia, Jay Whang, Ruiqi Gao, Alexey Gritsenko, Diederik~P Kingma, Ben Poole, Mohammad Norouzi, David~J Fleet, et~al.
\newblock Imagen video: High definition video generation with diffusion models.
\newblock \emph{arXiv preprint arXiv:2210.02303}, 2022.

\bibitem[Hochreiter and Schmidhuber(1997)]{hochreiter1997long}
Sepp Hochreiter and J{\"u}rgen Schmidhuber.
\newblock Long short-term memory.
\newblock \emph{Neural computation}, 9\penalty0 (8):\penalty0 1735--1780, 1997.

\bibitem[Huang et~al.(2021)Huang, Fu, Gao, Zhao, Roohani, Leskovec, Coley, Xiao, Sun, and Zitnik]{huang2021therapeutics}
Kexin Huang, Tianfan Fu, Wenhao Gao, Yue Zhao, Yusuf Roohani, Jure Leskovec, Connor Coley, Cao Xiao, Jimeng Sun, and Marinka Zitnik.
\newblock Therapeutics data commons: Machine learning datasets and tasks for drug discovery and development.
\newblock \emph{Advances in neural information processing systems}, 2021.

\bibitem[Hwang(1980)]{hwang1980laplace}
Chii-Ruey Hwang.
\newblock Laplace's method revisited: weak convergence of probability measures.
\newblock \emph{The Annals of Probability}, pages 1177--1182, 1980.

\bibitem[Isakov(2006)]{isakov2006inverse}
Victor Isakov.
\newblock \emph{Inverse problems for partial differential equations}, volume 127.
\newblock Springer, 2006.

\bibitem[Jain et~al.(2023)Jain, Deleu, Hartford, Liu, Hernandez-Garcia, and Bengio]{jain2023gflownets}
Moksh Jain, Tristan Deleu, Jason Hartford, Cheng-Hao Liu, Alex Hernandez-Garcia, and Yoshua Bengio.
\newblock Gflownets for ai-driven scientific discovery.
\newblock \emph{Digital Discovery}, 2\penalty0 (3):\penalty0 557--577, 2023.

\bibitem[Jin et~al.(2020)Jin, Barzilay, and Jaakkola]{hiervae}
Wengong Jin, Regina Barzilay, and Tommi Jaakkola.
\newblock Hierarchical generation of molecular graphs using structural motifs.
\newblock In \emph{International conference on machine learning}, pages 4839--4848. PMLR, 2020.

\bibitem[Karalias and Loukas(2020)]{karalias2020erdos}
Nikolaos Karalias and Andreas Loukas.
\newblock Erdos goes neural: an unsupervised learning framework for combinatorial optimization on graphs.
\newblock \emph{Advances in Neural Information Processing Systems}, 33:\penalty0 6659--6672, 2020.

\bibitem[Khalil et~al.(2017)Khalil, Dai, Zhang, Dilkina, and Song]{khalil2017learning}
Elias Khalil, Hanjun Dai, Yuyu Zhang, Bistra Dilkina, and Le~Song.
\newblock Learning combinatorial optimization algorithms over graphs.
\newblock \emph{Advances in neural information processing systems}, 30, 2017.

\bibitem[Kim et~al.(2023)Kim, Berto, Ahn, and Park]{kim2023bootstrapped}
Minsu Kim, Federico Berto, Sungsoo Ahn, and Jinkyoo Park.
\newblock Bootstrapped training of score-conditioned generator for offline design of biological sequences.
\newblock \emph{arXiv preprint arXiv:2306.03111}, 2023.

\bibitem[Kirkpatrick et~al.(1983)Kirkpatrick, Gelatt~Jr, and Vecchi]{kirkpatrick1983optimization}
Scott Kirkpatrick, C~Daniel Gelatt~Jr, and Mario~P Vecchi.
\newblock Optimization by simulated annealing.
\newblock \emph{science}, 220\penalty0 (4598):\penalty0 671--680, 1983.

\bibitem[Kloeden et~al.(1992)Kloeden, Platen, Kloeden, and Platen]{kloeden1992stochastic}
Peter~E Kloeden, Eckhard Platen, Peter~E Kloeden, and Eckhard Platen.
\newblock \emph{Stochastic differential equations}.
\newblock Springer, 1992.

\bibitem[Kong et~al.(2023)Kong, Kamarthi, Chen, Prakash, and Zhang]{kong2023uncertainty}
Lingkai Kong, Harshavardhan Kamarthi, Peng Chen, B~Aditya Prakash, and Chao Zhang.
\newblock Uncertainty quantification in deep learning.
\newblock In \emph{Proceedings of the 29th ACM SIGKDD Conference on Knowledge Discovery and Data Mining}, pages 5809--5810, 2023.

\bibitem[Kong et~al.(2024)Kong, Sun, Zhuang, Wang, Mu, and Zhang]{kong2024two}
Lingkai Kong, Haotian Sun, Yuchen Zhuang, Haorui Wang, Wenhao Mu, and Chao Zhang.
\newblock Two birds with one stone: Enhancing uncertainty quantification and interpretability with graph functional neural process.
\newblock In \emph{International Conference on Artificial Intelligence and Statistics}, pages 4582--4590. PMLR, 2024.

\bibitem[Kool et~al.(2018)Kool, van Hoof, and Welling]{kool2018attention}
Wouter Kool, Herke van Hoof, and Max Welling.
\newblock Attention, learn to solve routing problems!
\newblock In \emph{International Conference on Learning Representations}, 2018.

\bibitem[Krenn et~al.(2020)Krenn, H{\"a}se, Nigam, Friederich, and Aspuru-Guzik]{krenn2020self}
Mario Krenn, Florian H{\"a}se, AkshatKumar Nigam, Pascal Friederich, and Alan Aspuru-Guzik.
\newblock Self-referencing embedded strings (selfies): A 100\% robust molecular string representation.
\newblock \emph{Machine Learning: Science and Technology}, 1\penalty0 (4):\penalty0 045024, 2020.

\bibitem[Krishnamoorthy et~al.(2023)Krishnamoorthy, Mashkaria, and Grover]{DDOM}
Siddarth Krishnamoorthy, Satvik Mashkaria, and Aditya Grover.
\newblock Diffusion models for black-box optimization.
\newblock \emph{International Conference on Machine Learning}, 2023.

\bibitem[Kumar and Levine(2020)]{KumanMINs}
Aviral Kumar and Sergey Levine.
\newblock Model inversion networks for model-based optimization.
\newblock \emph{Advances in Neural Information Processing Systems}, 2020.

\bibitem[Landrum et~al.(2020)Landrum, Tosco, Kelley, sriniker, gedeck, NadineSchneider, Vianello, Ric, Dalke, Cole, AlexanderSavelyev, Swain, Turk, N, Vaucher, Kawashima, Wójcikowski, Probst, guillaume godin, Cosgrove, Pahl, JP, Berenger, strets123, JLVarjo, O'Boyle, Fuller, Jensen, Sforna, and DoliathGavid]{greg_landrum_2020_3732262}
Greg Landrum, Paolo Tosco, Brian Kelley, sriniker, gedeck, NadineSchneider, Riccardo Vianello, Ric, Andrew Dalke, Brian Cole, AlexanderSavelyev, Matt Swain, Samo Turk, Dan N, Alain Vaucher, Eisuke Kawashima, Maciej Wójcikowski, Daniel Probst, guillaume godin, David Cosgrove, Axel Pahl, JP, Francois Berenger, strets123, JLVarjo, Noel O'Boyle, Patrick Fuller, Jan~Holst Jensen, Gianluca Sforna, and DoliathGavid.
\newblock rdkit/rdkit: 2020\_03\_1 (q1 2020) release, March 2020.
\newblock URL \url{https://doi.org/10.5281/zenodo.3732262}.

\bibitem[Laporte et~al.(1992)Laporte, Louveaux, and Mercure]{laporte1992vehicle}
Gilbert Laporte, Francois Louveaux, and H{\'e}l{\`e}ne Mercure.
\newblock The vehicle routing problem with stochastic travel times.
\newblock \emph{Transportation science}, 26\penalty0 (3):\penalty0 161--170, 1992.

\bibitem[Larson et~al.(2019)Larson, Menickelly, and Wild]{larson2019derivative}
Jeffrey Larson, Matt Menickelly, and Stefan~M Wild.
\newblock Derivative-free optimization methods.
\newblock \emph{Acta Numerica}, 28:\penalty0 287--404, 2019.

\bibitem[Li and Malik(2016)]{li2016learning}
Ke~Li and Jitendra Malik.
\newblock Learning to optimize.
\newblock In \emph{International Conference on Learning Representations}, 2016.

\bibitem[Li et~al.(2023)Li, Guo, Wang, and Yan]{li2023distribution}
Yang Li, Jinpei Guo, Runzhong Wang, and Junchi Yan.
\newblock From distribution learning in training to gradient search in testing for combinatorial optimization.
\newblock \emph{Advances in Neural Information Processing Systems}, 2023.

\bibitem[Li et~al.(2024)Li, Kong, Du, Yu, Zhuang, Mu, and Zhang]{li2024muben}
Yinghao Li, Lingkai Kong, Yuanqi Du, Yue Yu, Yuchen Zhuang, Wenhao Mu, and Chao Zhang.
\newblock {MUB}en: Benchmarking the uncertainty of molecular representation models.
\newblock \emph{Transactions on Machine Learning Research}, 2024.
\newblock ISSN 2835-8856.

\bibitem[Li et~al.(2018)Li, Chen, and Koltun]{li2018combinatorial}
Zhuwen Li, Qifeng Chen, and Vladlen Koltun.
\newblock Combinatorial optimization with graph convolutional networks and guided tree search.
\newblock \emph{Advances in neural information processing systems}, 31, 2018.

\bibitem[Liuzzi et~al.(2016)Liuzzi, Lucidi, and Rinaldi]{liuzzi2016derivative}
Giampaolo Liuzzi, Stefano Lucidi, and Francesco Rinaldi.
\newblock A derivative-free approach to constrained multiobjective nonsmooth optimization.
\newblock \emph{SIAM Journal on optimization}, 26\penalty0 (4):\penalty0 2744--2774, 2016.

\bibitem[Lu et~al.(2023)Lu, Chen, Chen, Su, Li, and Zhu]{lu2023contrastive}
Cheng Lu, Huayu Chen, Jianfei Chen, Hang Su, Chongxuan Li, and Jun Zhu.
\newblock Contrastive energy prediction for exact energy-guided diffusion sampling in offline reinforcement learning.
\newblock In \emph{International Conference on Machine Learning}, pages 22825--22855. PMLR, 2023.

\bibitem[Ma et~al.(2019)Ma, Chen, Jin, Flammarion, and Jordan]{ma2019sampling}
Yi-An Ma, Yuansi Chen, Chi Jin, Nicolas Flammarion, and Michael~I Jordan.
\newblock Sampling can be faster than optimization.
\newblock \emph{Proceedings of the National Academy of Sciences}, 116\penalty0 (42):\penalty0 20881--20885, 2019.

\bibitem[Metropolis et~al.(1953)Metropolis, Rosenbluth, Rosenbluth, Teller, and Teller]{metropolis1953equation}
Nicholas Metropolis, Arianna~W Rosenbluth, Marshall~N Rosenbluth, Augusta~H Teller, and Edward Teller.
\newblock Equation of state calculations by fast computing machines.
\newblock \emph{The journal of chemical physics}, 21\penalty0 (6):\penalty0 1087--1092, 1953.

\bibitem[Min et~al.(2023)Min, Bai, and Gomes]{min2023unsupervised}
Yimeng Min, Yiwei Bai, and Carla~P Gomes.
\newblock Unsupervised learning for solving the travelling salesman problem.
\newblock \emph{arXiv preprint arXiv:2303.10538}, 2023.

\bibitem[Nguyen and Balasubramanian(2023)]{nguyen2023stochastic}
Anthony Nguyen and Krishnakumar Balasubramanian.
\newblock Stochastic zeroth-order functional constrained optimization: Oracle complexity and applications.
\newblock \emph{INFORMS Journal on Optimization}, 5\penalty0 (3):\penalty0 256--272, 2023.

\bibitem[Oksendal(2013)]{oksendal2013stochastic}
Bernt Oksendal.
\newblock \emph{Stochastic differential equations: an introduction with applications}.
\newblock Springer Science \& Business Media, 2013.

\bibitem[Petropoulos et~al.(2023)Petropoulos, Laporte, Aktas, Alumur, Archetti, Ayhan, Battarra, Bennell, Bourjolly, Boylan, et~al.]{petropoulos2023operationsresearchsurvey}
Fotios Petropoulos, Gilbert Laporte, Emel Aktas, Sibel~A Alumur, Claudia Archetti, Hayriye Ayhan, Maria Battarra, Julia~A Bennell, Jean-Marie Bourjolly, John~E Boylan, et~al.
\newblock Operational research: Methods and applications.
\newblock \emph{Journal of the Operational Research Society}, pages 1--195, 2023.

\bibitem[Pidstrigach(2022)]{pidstrigach2022score}
Jakiw Pidstrigach.
\newblock Score-based generative models detect manifolds.
\newblock \emph{Advances in Neural Information Processing Systems}, 35:\penalty0 35852--35865, 2022.

\bibitem[Poole et~al.(2022)Poole, Jain, Barron, and Mildenhall]{poole2022dreamfusion}
Ben Poole, Ajay Jain, Jonathan~T Barron, and Ben Mildenhall.
\newblock Dreamfusion: Text-to-3d using 2d diffusion.
\newblock \emph{arXiv preprint arXiv:2209.14988}, 2022.

\bibitem[Raginsky et~al.(2017)Raginsky, Rakhlin, and Telgarsky]{raginsky2017non}
Maxim Raginsky, Alexander Rakhlin, and Matus Telgarsky.
\newblock Non-convex learning via stochastic gradient langevin dynamics: a nonasymptotic analysis.
\newblock In \emph{Conference on Learning Theory}, pages 1674--1703. PMLR, 2017.

\bibitem[Rao(2019)]{rao2019engineering}
Singiresu~S Rao.
\newblock \emph{Engineering optimization: theory and practice}.
\newblock John Wiley \& Sons, 2019.

\bibitem[RDKit, online()]{rdkit}
RDKit, online.
\newblock {RDK}it: Open-source cheminformatics.
\newblock \url{http://www.rdkit.org}.
\newblock [Online; accessed 11-April-2013].

\bibitem[Roberts and Tweedie(1996)]{roberts1996exponential}
Gareth~O Roberts and Richard~L Tweedie.
\newblock Exponential convergence of langevin distributions and their discrete approximations.
\newblock \emph{Bernoulli}, pages 341--363, 1996.

\bibitem[Rubinstein and Kroese(2016)]{rubinstein2016simulation}
Reuven~Y Rubinstein and Dirk~P Kroese.
\newblock \emph{Simulation and the Monte Carlo method}.
\newblock John Wiley \& Sons, 2016.

\bibitem[Salimans and Ho(2021)]{salimans2021should}
Tim Salimans and Jonathan Ho.
\newblock Should ebms model the energy or the score?
\newblock In \emph{Energy Based Models Workshop-ICLR 2021}, 2021.

\bibitem[Sanchez-Lengeling and Aspuru-Guzik(2018)]{sanchez2018inverse}
Benjamin Sanchez-Lengeling and Al{\'a}n Aspuru-Guzik.
\newblock Inverse molecular design using machine learning: Generative models for matter engineering.
\newblock \emph{Science}, 361\penalty0 (6400):\penalty0 360--365, 2018.

\bibitem[Shahriari et~al.(2015)Shahriari, Swersky, Wang, Adams, and De~Freitas]{shahriari2015taking}
Bobak Shahriari, Kevin Swersky, Ziyu Wang, Ryan~P Adams, and Nando De~Freitas.
\newblock Taking the human out of the loop: A review of bayesian optimization.
\newblock \emph{Proceedings of the IEEE}, 104\penalty0 (1):\penalty0 148--175, 2015.

\bibitem[Shields et~al.(2021)Shields, Stevens, Li, Parasram, Damani, Alvarado, Janey, Adams, and Doyle]{shields2021bayesian}
Benjamin~J Shields, Jason Stevens, Jun Li, Marvin Parasram, Farhan Damani, Jesus I~Martinez Alvarado, Jacob~M Janey, Ryan~P Adams, and Abigail~G Doyle.
\newblock Bayesian reaction optimization as a tool for chemical synthesis.
\newblock \emph{Nature}, 590\penalty0 (7844):\penalty0 89--96, 2021.

\bibitem[Snoek et~al.(2012)Snoek, Larochelle, and Adams]{snoek2012practical}
Jasper Snoek, Hugo Larochelle, and Ryan~P Adams.
\newblock Practical bayesian optimization of machine learning algorithms.
\newblock \emph{Advances in neural information processing systems}, 25, 2012.

\bibitem[Sohl-Dickstein et~al.(2015)Sohl-Dickstein, Weiss, Maheswaranathan, and Ganguli]{sohl2015deep}
Jascha Sohl-Dickstein, Eric Weiss, Niru Maheswaranathan, and Surya Ganguli.
\newblock Deep unsupervised learning using nonequilibrium thermodynamics.
\newblock In \emph{International conference on machine learning}, pages 2256--2265. PMLR, 2015.

\bibitem[Song et~al.(2022)Song, Vahdat, Mardani, and Kautz]{song2022pseudoinverse}
Jiaming Song, Arash Vahdat, Morteza Mardani, and Jan Kautz.
\newblock Pseudoinverse-guided diffusion models for inverse problems.
\newblock In \emph{International Conference on Learning Representations}, 2022.

\bibitem[Song et~al.(2023)Song, Zhang, Yin, Mardani, Liu, Kautz, Chen, and Vahdat]{song2023loss}
Jiaming Song, Qinsheng Zhang, Hongxu Yin, Morteza Mardani, Ming-Yu Liu, Jan Kautz, Yongxin Chen, and Arash Vahdat.
\newblock Loss-guided diffusion models for plug-and-play controllable generation.
\newblock In \emph{International Conference on Machine Learning}, pages 32483--32498. PMLR, 2023.

\bibitem[Song and Ermon(2019)]{song2019generative}
Yang Song and Stefano Ermon.
\newblock Generative modeling by estimating gradients of the data distribution.
\newblock \emph{Advances in neural information processing systems}, 32, 2019.

\bibitem[Song et~al.(2020)Song, Sohl-Dickstein, Kingma, Kumar, Ermon, and Poole]{song2020score}
Yang Song, Jascha Sohl-Dickstein, Diederik~P Kingma, Abhishek Kumar, Stefano Ermon, and Ben Poole.
\newblock Score-based generative modeling through stochastic differential equations.
\newblock In \emph{International Conference on Learning Representations}, 2020.

\bibitem[Srinivas et~al.(2010)Srinivas, Krause, Kakade, and Seeger]{srinivas2010gaussian}
Niranjan Srinivas, Andreas Krause, Sham Kakade, and Matthias Seeger.
\newblock Gaussian process optimization in the bandit setting: no regret and experimental design.
\newblock In \emph{Proceedings of the 27th International Conference on International Conference on Machine Learning}, pages 1015--1022, 2010.

\bibitem[Stanczuk et~al.(2022)Stanczuk, Batzolis, Deveney, and Sch{\"o}nlieb]{stanczuk2022your}
Jan Stanczuk, Georgios Batzolis, Teo Deveney, and Carola-Bibiane Sch{\"o}nlieb.
\newblock Your diffusion model secretly knows the dimension of the data manifold.
\newblock \emph{arXiv preprint arXiv:2212.12611}, 2022.

\bibitem[Stephan et~al.(2017)Stephan, Hoffman, Blei, et~al.]{stephan2017stochastic}
Mandt Stephan, Matthew~D Hoffman, David~M Blei, et~al.
\newblock Stochastic gradient descent as approximate bayesian inference.
\newblock \emph{Journal of Machine Learning Research}, 18\penalty0 (134):\penalty0 1--35, 2017.

\bibitem[Sun and Yang(2023)]{sun2023difusco}
Zhiqing Sun and Yiming Yang.
\newblock {DIFUSCO}: Graph-based diffusion solvers for combinatorial optimization.
\newblock In \emph{Thirty-seventh Conference on Neural Information Processing Systems}, 2023.
\newblock URL \url{https://openreview.net/forum?id=JV8Ff0lgVV}.

\bibitem[Sutton et~al.(1999)Sutton, McAllester, Singh, and Mansour]{REINFORCE}
Richard~S Sutton, David McAllester, Satinder Singh, and Yishay Mansour.
\newblock Policy gradient methods for reinforcement learning with function approximation.
\newblock \emph{Advances in Neural Information Processing Systems}, 1999.

\bibitem[Tancik et~al.(2020)Tancik, Srinivasan, Mildenhall, Fridovich-Keil, Raghavan, Singhal, Ramamoorthi, Barron, and Ng]{tancik2020fourier}
Matthew Tancik, Pratul Srinivasan, Ben Mildenhall, Sara Fridovich-Keil, Nithin Raghavan, Utkarsh Singhal, Ravi Ramamoorthi, Jonathan Barron, and Ren Ng.
\newblock Fourier features let networks learn high frequency functions in low dimensional domains.
\newblock \emph{Advances in Neural Information Processing Systems}, 33:\penalty0 7537--7547, 2020.

\bibitem[Todorov et~al.(2012)Todorov, Erez, and Tassa]{mujoco}
Emanuel Todorov, Tom Erez, and Yuval Tassa.
\newblock Mujoco: A physics engine for model-based control.
\newblock \emph{IEEE/RSJ International Conference on Intelligent Robots and Systems}, pages 5026--5033, 2012.

\bibitem[Trabucco et~al.(2021{\natexlab{a}})Trabucco, Geng, Kumar, and Levine]{COMs}
Brandon Trabucco, Xinyang Geng, Aviral Kumar, and Sergey Levine.
\newblock Conservative objective models for effective offline model-based optimization.
\newblock \emph{International Conference on Machine Learning}, 2021{\natexlab{a}}.

\bibitem[Trabucco et~al.(2021{\natexlab{b}})Trabucco, Kumar, Geng, and Levine]{trabucco2021conservative}
Brandon Trabucco, Aviral Kumar, Xinyang Geng, and Sergey Levine.
\newblock Conservative objective models for effective offline model-based optimization.
\newblock In \emph{International Conference on Machine Learning}, pages 10358--10368. PMLR, 2021{\natexlab{b}}.

\bibitem[Trabucco et~al.(2022)Trabucco, Geng, Kumar, and Levine]{designbench}
Brandon Trabucco, Xinyang Geng, Aviral Kumar, and Sergey Levine.
\newblock Design-bench: Benchmarks for data-driven offline model-based optimization.
\newblock In \emph{International Conference on Machine Learning}, pages 21658--21676. PMLR, 2022.

\bibitem[Trillos et~al.(2023)Trillos, Hosseini, and Sanz-Alonso]{trillos2023optimization}
N~Garc{\'\i}a Trillos, B~Hosseini, and D~Sanz-Alonso.
\newblock From optimization to sampling through gradient flows.
\newblock \emph{NOTICES OF THE AMERICAN MATHEMATICAL SOCIETY}, 70\penalty0 (6), 2023.

\bibitem[Usmanova et~al.(2019)Usmanova, Krause, and Kamgarpour]{usmanova2019safe}
Ilnura Usmanova, Andreas Krause, and Maryam Kamgarpour.
\newblock Safe convex learning under uncertain constraints.
\newblock In \emph{The 22nd International Conference on Artificial Intelligence and Statistics}, pages 2106--2114. PMLR, 2019.

\bibitem[Vor{\v{s}}il{\'a}k et~al.(2020)Vor{\v{s}}il{\'a}k, Kol{\'a}{\v{r}}, {\v{C}}melo, and Svozil]{vorvsilak2020syba}
Milan Vor{\v{s}}il{\'a}k, Michal Kol{\'a}{\v{r}}, Ivan {\v{C}}melo, and Daniel Svozil.
\newblock Syba: Bayesian estimation of synthetic accessibility of organic compounds.
\newblock \emph{Journal of cheminformatics}, 12\penalty0 (1):\penalty0 1--13, 2020.

\bibitem[Wang et~al.(2022)Wang, Wu, Yang, Hao, and Li]{wang2022unsupervised}
Haoyu~Peter Wang, Nan Wu, Hang Yang, Cong Hao, and Pan Li.
\newblock Unsupervised learning for combinatorial optimization with principled objective relaxation.
\newblock \emph{Advances in Neural Information Processing Systems}, 35:\penalty0 31444--31458, 2022.

\bibitem[Watson et~al.(2023)Watson, Juergens, Bennett, Trippe, Yim, Eisenach, Ahern, Borst, Ragotte, Milles, et~al.]{watson2023novo}
Joseph~L Watson, David Juergens, Nathaniel~R Bennett, Brian~L Trippe, Jason Yim, Helen~E Eisenach, Woody Ahern, Andrew~J Borst, Robert~J Ragotte, Lukas~F Milles, et~al.
\newblock De novo design of protein structure and function with rfdiffusion.
\newblock \emph{Nature}, 620\penalty0 (7976):\penalty0 1089--1100, 2023.

\bibitem[Wenliang and Moran(2023)]{wenliang2023score}
Li~Kevin Wenliang and Ben Moran.
\newblock Score-based generative models learn manifold-like structures with constrained mixing.
\newblock \emph{arXiv preprint arXiv:2311.09952}, 2023.

\bibitem[Wibisono(2018)]{wibisono2018sampling}
Andre Wibisono.
\newblock Sampling as optimization in the space of measures: The langevin dynamics as a composite optimization problem.
\newblock In \emph{Conference on Learning Theory}, pages 2093--3027. PMLR, 2018.

\bibitem[Wilson et~al.(2016{\natexlab{a}})Wilson, Hu, Salakhutdinov, and Xing]{wilson2016stochastic}
Andrew~G Wilson, Zhiting Hu, Russ~R Salakhutdinov, and Eric~P Xing.
\newblock Stochastic variational deep kernel learning.
\newblock \emph{Advances in neural information processing systems}, 29, 2016{\natexlab{a}}.

\bibitem[Wilson et~al.(2016{\natexlab{b}})Wilson, Hu, Salakhutdinov, and Xing]{wilson2016deep}
Andrew~Gordon Wilson, Zhiting Hu, Ruslan Salakhutdinov, and Eric~P Xing.
\newblock Deep kernel learning.
\newblock In \emph{Artificial intelligence and statistics}, pages 370--378. PMLR, 2016{\natexlab{b}}.

\bibitem[Wu et~al.(2023)Wu, Trippe, Naesseth, Blei, and Cunningham]{wu2023practical}
Luhuan Wu, Brian~L Trippe, Christian~A Naesseth, David~M Blei, and John~P Cunningham.
\newblock Practical and asymptotically exact conditional sampling in diffusion models.
\newblock \emph{arXiv preprint arXiv:2306.17775}, 2023.

\bibitem[Yu et~al.(2021)Yu, Ahn, Song, and Shin]{yu2021roma}
Sihyun Yu, Sungsoo Ahn, Le~Song, and Jinwoo Shin.
\newblock Roma: Robust model adaptation for offline model-based optimization.
\newblock \emph{Advances in Neural Information Processing Systems}, 34:\penalty0 4619--4631, 2021.

\bibitem[Yuan et~al.(2023)Yuan, Chen, Liu, Neiswanger, and Liu]{yuan2023importance}
Ye~Yuan, Can Chen, Zixuan Liu, Willie Neiswanger, and Xue Liu.
\newblock Importance-aware co-teaching for offline model-based optimization.
\newblock In \emph{Thirty-seventh Conference on Neural Information Processing Systems}, 2023.

\bibitem[Zhang et~al.(2021)Zhang, Fu, Bengio, and Courville]{zhang2021unifying}
Dinghuai Zhang, Jie Fu, Yoshua Bengio, and Aaron Courville.
\newblock Unifying likelihood-free inference with black-box optimization and beyond.
\newblock In \emph{International Conference on Learning Representations}, 2021.

\bibitem[Zhang et~al.(2023)Zhang, Dai, Malkin, Courville, Bengio, and Pan]{zhang2023let}
Dinghuai Zhang, Hanjun Dai, Nikolay Malkin, Aaron Courville, Yoshua Bengio, and Ling Pan.
\newblock Let the flows tell: Solving graph combinatorial problems with {GF}lownets.
\newblock In \emph{Thirty-seventh Conference on Neural Information Processing Systems}, 2023.
\newblock URL \url{https://openreview.net/forum?id=sTjW3JHs2V}.

\bibitem[Zhuang et~al.(2023)Zhuang, Yu, Kong, Chen, and Zhang]{zhuang2023dygen}
Yuchen Zhuang, Yue Yu, Lingkai Kong, Xiang Chen, and Chao Zhang.
\newblock Dygen: Learning from noisy labels via dynamics-enhanced generative modeling.
\newblock In \emph{Proceedings of the 29th ACM SIGKDD Conference on Knowledge Discovery and Data Mining}, pages 3674--3686, 2023.

\end{thebibliography}

\appendix
\onecolumn
\newpage

\begin{center}
	{\Large \textbf{Appendix for \ours}}
\end{center}

\startcontents[sections]
\printcontents[sections]{l}{1}{\setcounter{tocdepth}{2}}

\section{Limitations and Future Work}
\label{appendix:limitations}
We discuss limitations and possible extensions of \ours. (i) \emph{Manifold preserving}. The guided diffusion may deviate from the manifold that the score network is trained, leading to error accumulations. One approach to mitigate this is to incorporate manifold constraints during the guided diffusion phase \citep{chung2022improving, he2023manifold}. (ii) \emph{Online learning}. We have applied \ours in the offline black-box optimization (BBO) setting. Considering the unknown constraints not only benefits the offline setting but also helps the online BBO. In the context of online BBO, we propose molecules and subsequently receive evaluations from the ground-truth simulator at each iteration to train the surrogate objective. Proposing a higher proportion of valid molecules can significantly increase the sampling efficiency of training the surrogate objective. Diffusion models also provide a probabilistic framework for estimating uncertainty in surrogate objective training.~\citep{kong2023uncertainty, kong2024two, li2024muben, zhuang2023dygen}.

\section{Broader Impacts}
\label{appendix:impact}
Optimization techniques can be used to solve a wide range of real-world problems, from decision making (planning, reasoning, and scheduling), to solving PDEs, and to designing new drugs and materials. The method we present in this paper extends the scope of the previous study to a more realistic setting where (partial) constraints for optimization problems are unknown, but we have access to samples from the feasible space. We expect that by learning the feasible set from data, our work can bring a positive impact to the community in accelerating solving real-world optimization problems and finding more realistic solutions. However, care should be taken to prevent the method from being used in harmful settings, such as optimizing drugs to enhance detrimental side effects.

\section{(Metropolis-adjusted) Langevin Dynamics.}
\label{sec:mala}

Langevin dynamics is a class of Markov Chain Monte Carlo (MCMC) algorithms that aims to generate samples from an unnormalized density $\pi(x)$ by simulating the differential equation
\begin{align}
    \textnormal{d}\mathbf{x}_t = \nabla_{\mathbf{x}}\log\pi(\mathbf{x}_t)\textnormal{d}t + \sqrt{2}\textnormal{d}\mathbf{w}_t.
\label{eq:appendix_langevin}
\end{align}
Theoretically, the continuous SDE of \cref{eq:appendix_langevin} is able to draw exact samples from $\pi(x)$. However, in practice, one needs to discretize the SDE using numerical methods such as the Euler-Maruyama method \citep{kloeden1992stochastic} for simulation. The Euler-Maruyama approximation of \Cref{eq:appendix_langevin} is given by
\begin{align}
    x_{t+\Delta t} = x_t + \nabla \log \pi(x) + \sqrt{2\Delta t}z, \qquad z \sim \mathcal{N}(0,1),
    \label{eq:appendix:discrete_langevin}
\end{align}
where $\Delta t$ is the step size. By drawing $x_0$ from an initial distribution and then simulating the dynamics in \cref{eq:appendix_langevin}, we can generate samples from $\pi(x)$ after a 'burn-in' period.
This algorithm is known as the Unadjusted Langevin Algorithm (ULA) \citep{roberts1996exponential}, which requires $\nabla \log \pi(x)$ to be $L$-Lipschitz for stability.

The ULA always accepts the new sample proposed by \cref{eq:appendix:discrete_langevin}. In contrast, to mitigate the discretization error when using a large step size, the Metropolis-adjusted Langevin Algorithm (MALA) \citep{grenander1994representations} uses the Metropolis-Hastings algorithm to accept or reject the proposed sample. Specifically, we first generate a proposed update $\hat x$ with \cref{eq:appendix_langevin}, then with probability $\text{min}(1, \frac{\pi(\hat x)\mathcal{N}(x|\hat x + \Delta t\cdot\nabla_{\hat x} \log \pi(\hat x), 2\Delta t)}{\pi(x)\mathcal{N}(\hat x|x + \Delta t\cdot\nabla_x \log \pi(x), 2\Delta t)}$), we set $x_{t+\Delta t}=\hat x$, otherwise $x_{t+\Delta}=x_t$. We provide the pseudocode of both algorithms in \cref{alg:mala}.

\begin{algorithm}
\caption{Sampling via the (Metropolis-Adjusted) Langevin dynamics}
\label{alg:mala}
\begin{algorithmic}[1] 
\REQUIRE unnormalized density $\pi(x)$, step size $\Delta t$
\STATE $x \sim \text{initial distribution}$
\STATE $\mathcal{X}=\varnothing$
\FOR{number of iterations}
\STATE $\hat x = x + \nabla_x \log \pi(x) \Delta t + \sqrt{2\Delta t}\cdot z, \;\; z \sim \mathcal{N}(0,1)$
\IF{applying Metropolis-Hastings test}
\STATE $u \sim \text{Uniform}[0,1]$
\STATE $\log P_{\text{accept}} = \log\frac{\pi(\hat x)\mathcal{N}(x|\hat x + \Delta t\cdot\nabla_{\hat x} \log \pi(\hat x), 2\Delta t)}{\pi(x)\mathcal{N}(\hat x|x + \Delta t\cdot\nabla_x \log \pi(x), 2\Delta t)}$
\STATE $\text{if }\;\; \log P_{\text{accept}} > \log u, \;\;\text{ then }\;\; x \gets \hat x$
\ELSE 
\STATE $x \leftarrow \hat x$
\ENDIF
\STATE $\mathcal{X} \gets \mathcal{X} \cup x$
\ENDFOR
\STATE \textbf{Return} $\mathcal{X}$
\end{algorithmic}
\end{algorithm}

\section{Pseudocode of the Two-Stage Sampling}
\label{sec:two-alg}

The pseudocode for the proposed  two-stage sampling method is provided in \Cref{alg:sampling}.

\begin{algorithm}[t]
\caption{Sampling via \ours for differentiable objective}
\label{alg:sampling}
\begin{algorithmic}[1] 
\REQUIRE inverse temperature schedule $\beta(t)$, diffusion volatility schedule $g(t)$ and drift $f(x, t)$, score model $s_\theta(x_t,t)$, energy function of the data distribution $E_{\theta}(x,t)$ if applying the MH correction.
\STATE $\mathcal{X} \gets \varnothing$
\STATE Sample $x_0\sim\mathcal{N}(0,\Id)$
\STATE \lightblue{// Stage I: \textit{Warm-up with guided diffusion.}}
\FOR{$t=0,\ldots,T$} 
\STATE Draw $z\sim\mathcal{N}(0,\Id)$, define $\tau = T-t$
\STATE $x_{t+\Delta t}\leftarrow x_t + [-f(x_t,\tau)+g^2(\tau)s_\theta(x_t,\tau) $
\STATE $\qquad \qquad\qquad\; -\beta(\tau)\nabla_{x_t} h(x_t)]\Delta t  + g(\tau)\sqrt{\Delta t} z$
\ENDFOR
\STATE \lightblue{// Stage II: \textit{Further correction with Langevin dynamics.}}
\FOR{$t = T,\ldots,T'$}
\STATE Draw $z\sim\mathcal{N}(0,\Id)$
\STATE $\hat x\gets x + [s_\theta(x, 0) - \beta \nabla h(x)]\Delta t + \sqrt{2\Delta t} z$
\IF{applying Metropolis-Hastings test}
\STATE $u \sim \text{Uniform}[0,1]$
\STATE $\ell_{\theta}(\hat x) = E_{\theta}(\hat x,0) - \beta h(\hat x)$
\STATE $\ell_{\theta}(x) = E_{\theta}(x,0) - \beta h(x)$
\STATE  $\ell(\hat x,x) = -\| x- \hat x - \Delta t [s_\theta(\hat x, 0) - \beta \nabla h(\hat x)]\|^2$
\STATE  $\ell(x,\hat x) = -\| \hat x- x - \Delta t [s_\theta(x, 0) - \beta \nabla h(x)]\|^2$
\STATE $\ell_{\text{acc}} = \ell_{\theta}(\hat x) - \ell_{\theta}(x)  + (\ell(\hat x,x) - \ell(x,\hat x))/(2\Delta t) $
\STATE $\text{if }\;\; \ell_{\text{acc}} > \log(u), \;\;\text{ then }\;\; x \leftarrow \hat x$
\ELSE 
\STATE $x \leftarrow \hat x$
\ENDIF
\STATE $\mathcal{X} \gets \mathcal{X} \cup \{x\}$
\ENDFOR
\STATE \textbf{Return} $\mathcal{X}$
\end{algorithmic}
\end{algorithm}

\section{Illustration of  Why Guided Diffusion Cannot Sample from the Product of Experts}
\label{sec:guided-prod}

The primary limitation of relying solely on Stage I is its inability to theoretically sample from our desired true target distribution, the product of distributions $\pi_\beta \propto p(x)q_{\beta}(x)$. This is because the score of the diffused marginal distribution does not directly correspond to the aggregate of the scores from each individual distribution,
\begin{align} 
\nabla_{x} \log \pi_{\beta}^t(x_t) &= \nabla_x \log\int p_{0}(x_0)q_{\beta}(x_0)p_{t|0}(x_t|x_0){\mathrm d}x_0\\ &\neq \nabla_x \log q_{\beta}(x_t)+ \nabla_x \log \underbrace{\int p_0(x_0)p_{t|0}(x_t|x_0)\mathrm{d}x_0,}_{p_t(x_t)}\quad t > 0, 
\end{align}
where $p_{t|0}(x_t|x_0)$ is the conditional density of the forward SDE starting at $x_0$.

Therefore we have to include Langevin dynamics (Stage II) as an optional stage to correct the bias introduced in Stage I. This stage can provide a theoretical guarantee for drawing exact samples from the product of distributions, despite the empirical observation that it offers only marginal performance improvements.

It's important to note that Stage I is essential because its output focuses on feasible minimizers under specific conditions, offering an improved initialization for Langevin dynamics. This has been demonstrated in our ablation study.

\section{Pseudocode of the Proposed Derivative-free Sampling }
\label{sec:snis}

The pseudocode of our derivative-free sampling algorithm is provided in \Cref{alg:snis}.

\begin{algorithm}[t]
\caption{Sampling via \ours for non-differentiable objective}
\label{alg:snis}
\begin{algorithmic}[1] 
\REQUIRE inverse temperature schedule $\beta(t)$, diffusion volatility schedule $g(t)$ and drift $f(x, t)$, score model $s_\theta(x_t,t)$.
\STATE \lightblue{// Initialization}
\STATE Sample $S$ particles $\{x_s^0\}_{s=1}^S$ from diffusion model
\STATE Compute $w_s^0 = \frac{\pi_{\beta}(x_s^0)}{p_{\theta}(x_s^0)} = q_{\beta}(x_s^0)$ for each particle
\STATE Normalize the weight $\tilde{w}_s^0=\frac{w_s^0}  {\sum_{s=1}^Mw_s^0}$ for each particle
\STATE Resample $S$ particles $\{x_s^0\}_{s=1}^S$ according to the weights
\STATE \lightblue{// Iterative Importance Sampling}
\FOR{$k=1,\cdots, K+1$}
\STATE Sample $t \sim \mathcal{U}[0,T]$
\FOR{$s=1,\cdots, S$}
\STATE \lightblue{// \(x_{s}^k\equiv x^k_{s,0}\)}
\STATE Add noise to the  particle by forward diffusion until time $t$:  \(x_{s,t}^{k-1} \sim p_{t|0}(x_{s,t}^{k-1} | x_{s,0}^{k-1})\)
\STATE Denoise the particle through backward diffusion: \(x_{s,0}^{k} \sim p_{0|t}(x_{s,0}^{k} | x_{s,t}^{k-1})\)
\STATE Sample \(x_{s,t,j}^{k-1} \sim p_{t|0}(x_{s,t}^{k-1} | x_{s,0}^{k-1}) = \mathcal{N}(x_{s,t}^{k-1}|\alpha_t x_{s,0}^{k-1}, \frac{\sigma_t^2}{1+\sigma_t}\Id)\) for $J$ times
\STATE \(Q(x_s^{k} | x_s^{k-1}) \approx \frac{1}{J} \sum_{j=1}^J \mathcal{N}\left(\frac{x_{s,t,j} + \sigma_t s_{\theta}(x_{s,t,j}^{k-1}, t)}{\alpha_t}, \frac{\sigma_t^2}{1+\sigma_t}\Id\right)\)
\STATE Marginalize to get $Q(x_s^k) \approx \frac{1}{S}\sum_{s=1}^S Q(x_s^k|x^{k-1}_s)$
\ENDFOR
\STATE Compute \( w_s^k = \frac{\pi_{\beta}(x_s^k)}{Q(x_s^{k})}\) for each particle
\STATE Normalize the weight for each particle: \(\tilde{w}_s^{k} = \frac{w_s^k}{\sum_{j=1}^S w_j^k} \)
\STATE Resample $M$ particles $\{x_s^k \}_{s=1}^S$ according to the weights
\ENDFOR
% \STATE $\mathcal{X} \gets \varnothing$
\STATE \textbf{Return} $\mathcal{X}=\{x_s^{K+1}\}_{s=1}^S $
\end{algorithmic}
\end{algorithm}

\section{Energy-based Parameterization}
\label{section:energy}

In a standard diffusion model, we learn the score of the data distribution directly as $s_{\theta}(x,t)=\nabla \log p_t(x)$. This parameterization can be used for ULA, which only requires gradients of log-likelihood. However, to incorporate the Metropolis-Hastings (MH) correction step, access to the unnormalized density of the data distribution is necessary to calculate the acceptance probability.

To enable the use of MH correction, we can instead learn the energy function of the data distribution, \ie~$p(x,t)\propto e^{E_{\theta}(x,t)}$. The simplest approach is to use a scalar-output neural network, denoted as $\text{NN}_{\theta}(x,t):\mathbb{R}^d\times \mathbb{R}\rightarrow \mathbb{R}$, to 
parameterize $E_{\theta}(x,t)$. By taking the gradient of this energy function with respect to the input $x$, we can derive the score of the data distribution. However, existing works have shown that this parameterization can cause difficulties during model training \citep{salimans2021should}. Following the approach by \citep{du2023reduce}, we define the energy function as $E_{\theta}(x,t) = -\frac{1}{2}|\text{NN}_{\theta}(x,t)|^2_2$, where $\text{NN}_{\theta}(x,t)$ is a vector-output neural network mapping from $\mathbb{R}^d\times\mathbb{R}$ to $\mathbb{R}$. Consequently, the score of the data distribution is represented as $s_{\theta}(x,t) = -\text{NN}_{\theta}(x,t) \nabla_x \text{NN}_{\theta}(x,t)$.

\section{End Distribution of the Warm-up Stage}
\label{sec:appendix_end_distribution}

In this section, we study in further detail the warm-up stage of \ours. We recall that we consider a process of the following form 
\begin{equation}
\label{eq:warm-up-process}
    \rmd \bfy_t^\beta = [-f(\bfy_t^\beta, 1-t) + g(1-t)^2 \nabla \log p_{1-t}(\bfy_t^\beta) - \beta(1-t) \nabla h(\bfy_t^\beta)] \rmd t + g(1-t) \rmd \bfw_t , \qquad \bfy_0^\beta \sim \bfx_T
\end{equation}
where $T=1$ and $p_t$ is the density w.r.t. Lebesgue measure of the distribution of $\bfx_t$ where 
\begin{equation}
\label{eq:forward_process}
    \rmd \bfx_t = f(t, \bfx_t) \rmd t + g(t) \rmd \bfw_t , \qquad \bfx_0 \sim p_0 .
\end{equation}
We recall that under mild assumption \citep{cattiaux2023time}, we have that $(\hat{\bfy}_t)_{t \in [0,1]} = (\bfx_{1-t})_{t \in [0,1]}$ satisfies
\begin{equation}
    \rmd \hat{\bfy}_t = [-f(\hat{\bfy}_t, 1-t) + g(1-t)^2 \nabla \log p_{1-t}(\hat{\bfy }_t)] \rmd t + g(1-t) \rmd \bfw_t , \qquad \hat{\bfy}_0 = \bfx_T . 
\end{equation}
Let us highlight some differences between \eqref{eq:warm-up-process} and the warm-up process described in Algorithm \ref{alg:sampling}. First, we note that we do not consider an approximation of the score but the real score function $\nabla \log p_t$. In addition, we do not consider a discretization of \eqref{eq:warm-up-process}. This difference is mainly technical. The discretization of diffusion processes is a well-studied topic, and we refer to \citep{de2021diffusion,benton2023linear,conforti2023score,chen2022sampling} for a thorough investigation. Our contribution to this work is orthogonal as we are interested in the role of $h$ on the distribution. Our main result is \Cref{prop:concentration} and details how the end distribution of the warm-up process concentrates on the minimizers of $h$, which also support the data distribution $p_0$.

We first show that under assumptions on $h$, $q_T^\beta$, the density w.r.t. the Lebesgue measure of $\bfy_T$ has the same support as $p_0$. We denote by $\supp(p_0)$ the support of $p_0$. We consider the following assumption.
\begin{assumption}
\label{assumption:manifold}
We have that for any $t \in [0,1]$, $g(t) = g(0)$ and $f(t,x) = -\gamma_0 x$ with $\gamma_0 > 0$. Assume that $p_0$ has bounded support, i.e. there exists $R > 0$ such that $\supp(p_0) \subset \mathrm{B}(0, R)$ with $\mathrm{B}(0,R)$ the ball with center $0$. In addition, assume that $h$ is Lipschitz.
\end{assumption}

Then, we have the following proposition.

\begin{proposition}
\label{prop:concentration_manifold}
    Assume \textup{\textbf{A\ref{assumption:manifold}}}. Then, we have that for any $\beta > 0$, $\bfy_T^\beta \in \supp(p_0)$.
\end{proposition}

\begin{proof}
    This directly applies to the results of \citep{pidstrigach2022score}. First, we have that \citep[Assumption 1, Assumption 2]{pidstrigach2022score} are satisfied using A\ref{assumption:manifold}, \citep[Lemma 1]{pidstrigach2022score} and the second part of \citep[Theorem 2]{pidstrigach2022score}. We conclude using the first part of \citep[Theorem 2]{pidstrigach2022score}.
\end{proof}

In \Cref{prop:concentration_manifold}, we show that the guided reconstructed scheme used for warm-up \eqref{eq:warm-up-process} cannot discover minimizers outside the support of $p_0$. In \Cref{prop:concentration}, we will show that we concentrate on the minimizers inside the support of $p_0$ under additional assumptions.

Next, we make the following assumption, which is mostly technical. We denote $q_t^\beta$ the distribution of $\bfy_t^\beta$ for any $t \in [0,1]$. We also denote $(p_{1-t}^\beta)_{t \in [0,1]} = (q_t^\beta)_{t \in [0,1]}$.

\begin{assumption}
    \label{assumption:score_control}
We have that $h \in \rmc^\infty(\rset^d, \rset)$. In addition, $C > 0$ exists such that for any $x \in \rset^d$, we have a.s. 
\begin{equation}
\label{eq:control_w}
    \textstyle \abs{\int_0^1 \beta_t(W_t^\beta(\bfz_t) - W_0^\beta(\bfz_0)) \rmd t } \leq C ,
\end{equation}
with $\rmd \bfz_t = \{f_t - g_t^2 \nabla \log p_t \rangle \rmd t + \rmd \bfw_t$, $\bfz_0 = x$, $W_t = \langle \nabla \log p_t^\beta, \nabla h \rangle + \Delta h$ and $\beta_t \equiv \beta(1-t)$. For $t \in [0,1]$ and $x \in \rset^d$
\begin{equation}
V_t = \mathrm{div}(f_{t} - g_{t}^2 \nabla \log p_{t}) , \qquad W_t^\beta = \langle \nabla h, \nabla \log p_t^\beta \rangle + \Delta h . 
    \end{equation}
    Assume that $V$ and $W^\beta$ are continuous and bounded on $[0,1] \times \rset^d$. Assume that $(p_t)_{t \in [0,1]}$ and $(p_t^\beta)_{t \in [0,1]}$ are strong solutions to their associated Fokker-Planck equations.
\end{assumption}

We do not claim that we verify this hypothesis in this paper. Proving A\ref{assumption:score_control} is out of the scope of this work, and we mainly use it to 1) control high-order terms and 2) provide sufficient conditions to apply the Feynman-Kac theorem \citep[Theorem 7.13]{oksendal2013stochastic} and the Fokker-Planck equation. The bound in \eqref{eq:control_w} controls the regularity of the $W_t^\beta(\bfz_t)$. Given that (under some mild regularity assumption), $(W_t^\beta(\bfz_t))_{t  \in [0,1]}$ satisfies a Stochastic Differential Equation we expect that $\mathbb{E}[\| (W_t^\beta(\bfz_t) - W_0^\beta(\bfz_0)) \|] \leq C \sqrt{t}$ for any $t \in [0,1]$ and some constant $C > 0$ (independent of $t$). Therefore, we get that \eqref{eq:control_w} is true in expectation under some regularity assumption ( a. $\Delta h$  is Lipschitz. bThe diffused distribution $p_t$ 
 is smooth with bounded derivatives.). We conjecture that the almost sure bound we require is unnecessary and that moment bounds should be enough. We leave this study for future work.

\begin{proposition}
\label{prop:control_warm_up}
    Assume \textup{\textbf{A\ref{assumption:score_control}}}. For any $x \in \rset^d$, let $\tilde{p}_0^\beta(x)$ be given by 
    \begin{equation}
        \tilde{p}_0^\beta(x) = p_0(x) \exp[\log(\beta_0)\{ \Delta h(x) + \langle \nabla \log p_0^\beta(x), \nabla h(x) \rangle \}] ,
    \end{equation}
    where $p_0^\beta$ is the distribution of $\bfy_T^\beta$ and $\beta_0$ is the inverse temperature at the end of the process.  
    
    Then there exists $C_0 >0$ such that for any $x \in \rset^d$
    \begin{equation}
        (1/C_0) \tilde{p}_0^\beta(x) \leq p_0^\beta(x) \leq C_0 \tilde{p}_0^\beta(x) .
    \end{equation}
\end{proposition}

%Before diving into the proof, let us provide some insight on \cref{prop:control_warm_up}. 

%The distribution of $p_0^\beta$ is the distribution of interest as it is the output distribution of the guided warm-up process. 
%We consider a specific schedule $\beta(t) = 1/t$ with $t > t_0$ and $0$ otherwise to obtain meaningful bounds. This means that as $\beta_0 \to \infty$, we are increasingly emphasizing $t \approx 0$. 

\begin{proof}
    First, for any $t \in [0,1]$ we denote $p_t$ the density of $\bfx_t$ where $(\bfx_t)_{t \in [0,1]}$ is given by \eqref{eq:forward_process}. Similarly, we denote $q_t^\beta$ the distribution of $\bfy_t^\beta$ for any $\tau_0 > 0$ and $t \in [0,1]$. Finally, we denote $(p_{1-t}^\beta)_{t \in [0,1]} = (q_t^\beta)_{t \in [0,1]}$ for any $\tau_0 >0$. In what follows, we fix $t_0 > 0$. Using A\ref{assumption:score_control} we have that for any $t \in [0,1]$ 
    \begin{equation}
        \partial_t p_t = -\mathrm{div}(f_t p_t) + (g_t^2 /2) \Delta p_t .
    \end{equation}
    Therefore, we have that for any $t \in [0,1]$
    \begin{equation}
        \partial_t p_t + \mathrm{div}(f_t p_t) - (g_t^2 /2) \Delta p_t = 0.
    \end{equation}
    This can also be rewritten as 
    \begin{equation}
    \label{eq:fokker_planck}
        \partial_t p_t + \langle f_t - g_t^2 \nabla \log p_t, \nabla p_t \rangle + (g_t^2 /2) \Delta p_t + \mathrm{div}(f_t -g_t^2 \nabla \log p_t) p_t = 0 , 
    \end{equation}
    where we have used that $\mathrm{div}(\nabla \log p_t p_t) = \Delta p_t$.
    Similarly, we have that 
    \begin{equation}
        \partial_t q_t^\beta = -\mathrm{div}(\{-f_{1-t} + g_{1-t}^2 \nabla \log p_{1-t} - \beta_{1-t} \nabla h\}q_t^\beta) + (g_{1-t}^2 /2) \Delta q_t^\beta .
    \end{equation}
    We have that 
    \begin{equation}
        \partial_t p_t^\beta = \mathrm{div}(\{-f_{t} + g_{t}^2 \nabla \log p_{t} - \beta_{t}\nabla h\}p_t^\beta) - (g_{t}^2 /2) \Delta p_t^\beta .
    \end{equation}
    This can also be rewritten as 
    \begin{equation}
    \label{eq:foker_planck_tau}
        \partial_t p_t^\beta + \langle f_t - g_t^2 \nabla \log p_t^\beta, \nabla p_t^\beta \rangle + \mathrm{div}(f_{t} - g_{t}^2 \nabla \log p_{t}) p_t^\beta + (g_{t}^2 /2) \Delta p_t^\beta + \mathrm{div}(\beta_t \nabla hp_t^\beta) = 0 .
    \end{equation}
    Finally, this can be rewritten as 
    \begin{equation}
    \label{eq:}
        \partial_t p_t^\beta + \langle f_t - g_t^2 \nabla \log p_t, \nabla p_t^\beta \rangle  + (g_{t}^2 /2) \Delta p_t^\beta + \mathrm{div}(f_{t} - g_{t}^2 \nabla \log p_{t}) p_t^\beta + 
        \beta_t\{ \langle \nabla h, \nabla \log p_t^\beta \rangle + \Delta h\} p_t^\beta  = 0.
    \end{equation}
    In what follows, we denote 
    \begin{equation}
    \label{eq:notation_feynman_kac}
        \mu_t = f_t - g_t^2 \nabla \log p_t , \qquad V_t = \mathrm{div}(f_{t} - g_{t}^2 \nabla \log p_{t}) , \qquad W_t^\beta = \langle \nabla h, \nabla \log p_t^\beta \rangle + \Delta h . 
    \end{equation}
    Hence, using \eqref{eq:foker_planck_tau} we have 
    \begin{equation}
    \label{eq:fokker_planck_tau_revisited}
        \partial p_t^\beta + \langle \mu_t, \nabla \log p_t^\beta \rangle + (g_t^2 /2) \Delta p_t^\beta + V_t p_t^\beta + \beta_t W_t^\beta p_t^\beta = 0 . 
    \end{equation}
    Similarly, using \eqref{eq:fokker_planck} we have 
   \begin{equation}
   \label{eq:fokker_planck_revisited}
        \partial p_t + \langle \mu_t, \nabla \log p_t \rangle + (g_t^2 /2) \Delta p_t + V_t p_t = 0 . 
    \end{equation}
    Therefore, combining \eqref{eq:fokker_planck_revisited}, A\ref{assumption:score_control} and \citep[Theorem 7.13]{oksendal2013stochastic} we get that for any $x \in \rset^d$ 
    \begin{equation}
       \textstyle  p_0(x) = \mathbb{E}[ \exp[\int_0^1 V_t(\bfz_t) \rmd t ] p_T(\bfz_T) \ | \ \bfz_0 = x] ,
    \end{equation}
    with $\rmd \bfz_t = \mu_t \rmd t + g_t \rmd \bfw_t$ and $\bfz_0 = x$. 
    Similarly, combining \eqref{eq:fokker_planck_tau_revisited}, A\ref{assumption:score_control} and \citep[Theorem 7.13]{oksendal2013stochastic} we get that for any $x \in \rset^d$ 
    \begin{equation}
    \label{eq:feynman-kac-tau}
       \textstyle  p_0^\beta(x) = \mathbb{E}[ \exp[\int_0^1 V_t(\bfz_t) \rmd t] \exp[ \int_0^1 \beta_t W_t^\beta(\bfz_t) \rmd t ] p_T(\bfz_T) \ | \ \bfz_0 = x] .
    \end{equation}
    Using \eqref{eq:notation_feynman_kac}, we have that 
    \begin{equation}
       \textstyle \int_0^1 \beta_t W_t^\beta(\bfz_t) \rmd t = \int_{0}^1 \beta_t \{ \langle \nabla h, \nabla \log p_t^\beta \rangle + \Delta h \}(\bfz_t) \rmd t .
    \end{equation}
    Hence, we have
   \begin{equation}
       \textstyle \int_0^1 \beta_t 
    W_t^\beta(\bfz_t) \rmd t = \log(\beta_0) \{ \langle \nabla h, \nabla \log p_0^\beta \rangle + \Delta h \}(\bfz_0) + \int_{0}^1\beta_t(W_t^\beta(\bfz_t) - W_0^\beta(\bfz_0)) \rmd t .
    \end{equation}
    Using A\ref{assumption:score_control} we have that 
    \begin{equation}
        -C \leq  \textstyle \int_0^1 \beta_tW_t^\beta(\bfz_t) \rmd t - \log(\beta_0) \{ \langle \nabla h, \nabla \log p_0^\beta \rangle + \Delta h \}(\bfz_0) \leq C\,.
    \end{equation}
    Hence,
    \begin{equation}
        -C + \log(\beta_0) \{ \langle \nabla h, \nabla \log p_0^\beta \rangle + \Delta h \}(\bfz_0) \leq \textstyle \int_0^1 \beta_tW_t^\beta(\bfz_t) \rmd t \leq C + \log(\beta_0) \{ \langle \nabla h, \nabla \log p_0^\beta \rangle + \Delta h \}(\bfz_0)\,.
    \end{equation}
    Combining this result with \eqref{eq:feynman-kac-tau} we get that for any $x \in \rset^d$
    \begin{align}
        \textstyle  p_0^\beta(x) &\leq \textstyle \mathbb{E}[ \exp[\int_0^1 V_t(\bfz_t) \rmd t]  p_T(\bfz_T) \ | \ \bfz_0 = x] \exp[\log(\beta_0) \{ \langle \nabla h, \nabla \log p_0^\beta + \Delta h \rangle\}(x)] \exp[C] \\
        &= p_0(x) \exp[\log(\beta_0) \{ \langle \nabla h, \nabla \log p_0^\beta + \Delta h \rangle\}(x)] \exp[C] . 
    \end{align}
    Similarly, we have for any $x \in \rset^d$
    \begin{equation}
        \textstyle  p_0^\beta(x) \geq p_0(x) \exp[\log(\beta_0) \{ \langle \nabla h, \nabla \log p_0^\beta \rangle + \Delta h\}(x)] \exp[-C] ,
    \end{equation}
    which concludes the proof.
\end{proof}

In \Cref{prop:control_warm_up}, we show that the output distribution of the warm-up process is upper and lower bounded by a product of experts comprised of 
\begin{enumerate*}[label=(\roman*)]
    \item $p_0$ which ensures the \emph{feasibility} conditions
    \item $\exp[\log(\beta_0) W_0^\beta]$ related to the \emph{optimization} of the objective. 
\end{enumerate*}
While \Cref{prop:control_warm_up} gives an explicit form for $p_0^\beta$, it does not provide insights on the properties of this distribution. However, we can still infer some limiting properties.

\begin{proposition}
\label{prop:concentration}
    If $x^\star$ is a local strict minimizer of $h$ in the support of $p_0$ then $\lim_{\beta_0 \to \infty} p_0^\beta(x^\star) = +\infty$. If $x^\star$ is a local strict minimizer of $h$ not in the support of $p_0$ then $\lim_{\beta_0 \to \infty} p_0^\beta(x^\star) = 0$. 
\end{proposition}

\begin{proof}
The case where $x^\star$ is not in support of $p_0$ is trivial using \Cref{prop:control_warm_up}. We now assume that $x^\star$ is in the support of $p_0$. 
    Using \Cref{prop:control_warm_up} and that $\nabla h(x^\star) = 0$ and $\Delta h(x^\star) > 0$ since the local minimizer $x^\star$ is strict, we get that $\lim_{\beta_0 \to \infty} p_0^\beta(x^\star) = +\infty$, which concludes the proof.
\end{proof}

In particular, \Cref{prop:concentration} shows that the limit distribution of $\bfy_T^\beta$ concentrates around the minimizers of $h$, which is the expected behavior of increasing the inverse temperature. What is also interesting is that we only target the minimizers of $h$, which are inside the support of $p_0$.  This is our primary goal, which is \emph{constrained} optimization of $h$.

\section{Experiment Details}
\label{appendix:exp}

\subsection{Computing Infrastructure}
\label{appendix:computing}
System: Ubuntu 18.04.6 LTS; Python 3.9; Pytorch
1.11. CPU: Intel(R) Xeon(R) Silver 4214 CPU @ 2.20GHz. GPU: GeForce GTX 2080 Ti.

\subsection{Synthetic Branin Experiment}

\begin{wrapfigure}{r}{0.3\textwidth}
  \begin{center}  \includegraphics[width=0.3\textwidth]{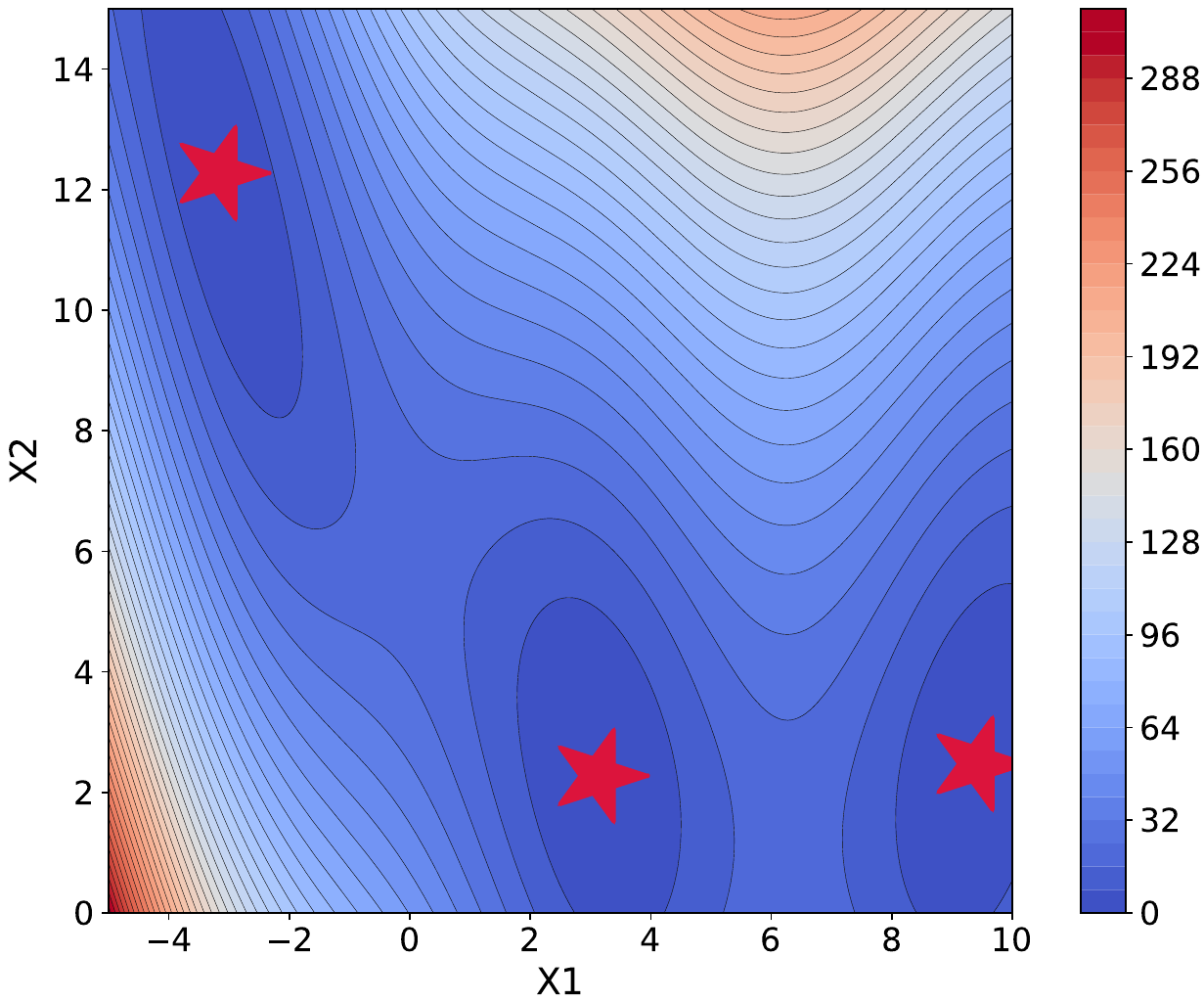}
  \end{center}
  %\vspace{-1.5em}
  \caption{Branin function.}
  \label{fig:branin}
  %\vspace{-1.2em}
\end{wrapfigure}
We consider the commonly used Branin function as a synthetic toy example that takes the following form (illustrated in Figure~\ref{fig:branin}):

\begin{equation}
    f(x_1,x_2)=a(x_2-bx_1^2+cx_1-r)^2+s(1-t)\text{cos}(x_1)+s , 
\end{equation}
where $a=1$, $b=\frac{5.1}{4\pi^2}$, $c=\frac{5}{\pi}$, $r=6$, $s=10$, $t=\frac{1}{8\pi}$.

The Branin function $f(x_1,x_2)$ has three global minimas located at points $(-\pi,12.275)$, $(\pi,2.275)$, and $(9.42478,2.475)$  with  a value of $0.397887$.

\textbf{Dataset details.} We curate $6,000$ data points by sampling uniformly in an ellipse region with center $(-0.2, 7.5)$ and semi-axis lengths $(3.6, 8.0)$ as training data (the blue region in \cref{fig:landscape_unknown}). It is tilted counterclockwise by 25 degrees, ensuring that it covers two minimizers of the function: $(-\pi,12.275)$ and $(\pi,2.275)$. It is worth noting that sampling points $(x_1,x_2)$ to construct the training dataset are irrelevant to the objective value $f(x_1,x_2)$.  For the experiment with additional known constraints, we introduce two constraints $x_2 \leq \frac{3}{2} x_1 +\frac{15}{2}$ and $x_2 \leq -\frac{3}{2} x_1 +15$ (the pink region in \cref{fig:landscape_unknown}) to further narrow down the feasible solution to $(\pi,2.275)$. We split the dataset into training and validation sets by 9:1. 

\textbf{Implementation details.}
We build the score network $s_\theta$ of the diffusion model with a 2-layer MLP architecture with 1024 hidden dimensions and ReLU activation function. The forward process is a Variance Preserving (VP) SDE \citep{song2020score}.
We set the minimum and maximum values of noise variance to be $0.01$ and $2.0$, respectively. We employ a fixed learning rate of $0.001$, a batch size of $128$, and $1000$ epochs for model training. At test time, we sample $500$ candidate solutions. We use a constant inverse temperature $\beta=5$ for the Boltzmann distribution induced by the objective function. For the distribution induced by the additional known constraints, we set $\beta'=10.$ 

\subsection{Offline Black-box Optimization}

\textbf{Dataset details.} DesignBench \citep{designbench} is an offline black-box optimization benchmark for real-world optimization tasks. Following \citep{DDOM}, we use three continuous tasks: Superconductor, D'Kitty Morphology and Ant Morphology, and three discrete tasks: TFBind8, TFBind10, and ChEMBL.
Consistent with \citep{DDOM}, we exclude NAS due to its significant computational resource demands. We also exclude Hopper as it is known to be buggy (see Appendix C in \citep{DDOM}). We split the dataset into training and validation sets by 9:1. 

\begin{itemize}
\item Superconductor: materials optimization. This task aims to search for materials with high critical temperatures. The dataset contains 17,014 vectors with 86 components representing the number of atoms of each chemical element in the formula. The provided oracle function is a pre-trained random forest regression model.
\item D'Kitty Morphology: robot morphology optimization. This task aims to optimize the parameters of a D'Kitty robot, such as size, orientation, and location of the limbs, to make it suitable for a specific navigation task. The dataset size is 10,004, and the parameter dimension is 56. It uses MuJoCO \citep{mujoco}, a robot simulator, as the oracle function.
\item Ant Morphology: robot morphology optimization. Similar to D'Kitty, this task aims to optimize the parameters of a quadruped robot to move as fast as possible. It consists of 10,004 data, and the parameter dimension is 60. It also uses MuJoCo as the oracle function.
\item TFBind8: DNA sequence optimization. This task aims to find the DNA sequence of length eight with the maximum binding affinity with transcription factor SIX6\_REF\_R1. The design space is the space of sequences of nucleotides represented as categorical variables. The size of the dataset is 32,898, with a dimension of 8. The ground truth serves as a direct oracle since the affinity for the entire design space is available.
\item TFBind10: DNA sequence optimization. Similar to TFBind8, this task aims to find the DNA sequence of length ten that has the maximum binding affinity with transcription factor SIX6\_REF\_R1. The design space consists of all possible designs of nucleotides. The size of the dataset is 10,000, with a dimension of 10. Since the affinity for the entire design space is available, it uses the ground truth as a direct oracle.
\item ChEMBL: molecule activity optimization. This task aims to find molecules with a high MCHC value when paired with assay CHEMBL3885882. The dataset consists of 441 samples of dimension 31. 
\end{itemize}

\textbf{Baselines.} We compare with eight baselines on DesignBench tasks. The results of all the baselines are from \citep{DDOM}. \textbf{Gradient ascent} learns a surrogate model of the objective function and generates the optimal solution by iteratively performing gradient ascent on the surrogate model. \textbf{CbAS} learns a density model in the design space coupled with a surrogate model of the objective function. It iteratively generates samples and refines the density model on the new samples during training.  \textbf{GP-qEI} fits a Gaussian Process on the offline dataset. It employs the quasi-Expected-Inprovement (qEI) acquisition function from Botorch \citep{botorch} for Bayesian optimization. \textbf{MINS} learns an inverse map from objective value back to design space using a Generative Adversarial Network (GAN). It then obtains optimal solutions through conditional generation. \textbf{REINFORCE} parameterizes a distribution over the design space and adjusts this distribution in a  direction that maximizes the efficacy of the surrogate model. \textbf{COMS} learns a conservative 
surrogate model by regularizing the adversarial samples.
It then utilizes gradient ascent to discover the optimal solution. \textbf{CMAES} enhances a distribution over the optimal design by adapting the covariance matrix according to the highest-scoring samples selected by the surrogate model. \textbf{DDOM} learns a conditional diffusion model to learn an inverse mapping from the objective value to the input space.

\textbf{Implementation details.} We build the score network $s_\theta$ using a simple feed-forward network. This network consists of two hidden layers, each with a width of 1024 units, and employs ReLU as the activation function. The forward process is a  Variance Preserving (VP) SDE \citep{song2020score}. We set the noise variance limits to a minimum of 0.01 and a maximum of 2.0. 

For the surrogate models, we explore various network architectures tailored to different datasets, including Long short-term memory (LSTM) \citep{hochreiter1997long}, Gaussian Fourier Network, and Deep Kernel Learning (DKL) \citep{wilson2016deep, wilson2016stochastic}. LSTM network uses a single-layer LSTM unit with a hidden dimension of 1024, followed by 1 hidden layer with a dimension of 1024, utilizing ReLU as the activation function. In the Gaussian Fourier regressor, Gaussian Fourier embeddings \citep{tancik2020fourier} are applied to the inputs $x$ and $t$. These embeddings are then processed through a feed-forward network with 3 hidden layers, each of 1024 width, utilizing Tanh as the activation function. This regressor is time-dependent, and its training objective follows the method used by \citep{song2020score} for training time-dependent classifiers in conditional generation. For DKL,  we use the ApproximateGP class in Gpytorch\footnote{\url{https://docs.gpytorch.ai/en/stable/_modules/gpytorch/models/approximate_gp.html\#ApproximateGP}}, which consists of a deep feature extractor and a Gaussian process (GP). The feature extractor is a simple feed-forward network consisting of $2$ hidden layers with a width of $500$ and $50$, respectively, and ReLU activations. The GP uses
radial basis function (RBF) kernel.

We use a fixed learning rate of 0.001 and a batch size of 128 for both the diffusion and surrogate models. During testing, we follow the evaluation protocol from the \citep{DDOM}, sampling 256 candidate solutions. We apply different annealing strategies for different datasets. Specifically, we apply exponential annealing for TFBind8, superconductor, D'Kitty, and ChEMBL. The exponential annealing strategy is defined as $\beta(\tau)= \beta_{\rm max}[1-\exp(-100(T-\tau))]$, where $\tau=T-t$, and a constant $\beta$ for Ant and TFBind10. Though exponential annealing usually exhibits better performance, we leave the exploration of exponential annealing on TFBind10 and D'Kitty for future work due to time limit. The step size $\Delta t$ is $0.001$ for the first stage, and $0.0001$ for the second stage.

Detailed hyperparameters and network architectures for each dataset are provided in \cref{appendix:hyperparameter}.

\begin{table}[h]
\centering
\begin{tabular}{@{}lccc@{}}
\toprule
               & Annealing strategy & $\beta_{\rm max}$ & Surrogate model\\ \midrule
TFBind8        &  Exponential                  & $200$    &  Gaussian Fourier\\
TFBind10       &   Constant                 &  $20$  & Gaussian Fourier \\
Superconductor &    Exponential                &  $100$  & Gaussian Fourier  \\
Ant            &  Exponential   &  30       & Gaussian Fourier   \\
D'Kitty        &   Constant                 & $3e4$  & DKL  \\
ChEMBL         &      Exponential               &  $100$  &  LSTM \\ \bottomrule
\end{tabular}
\caption{Implementation details on design-bench. $\beta_{\rm max}$ is the value of $\beta$ at the end of the annealing process.}
\label{appendix:hyperparameter}
\end{table}

\subsection{Multi-objective Molecule Optimization}

\textbf{Dataset details.} We curate the dataset by encoding 10000 molecules (randomly selected) from the ChEMBL dataset \citep{gaulton2012chembl} with HierVAE \citep{hiervae}, a commonly used molecule generative model based on VAE, which takes a hierarchical procedure in generating molecules by building blocks. Since validity is important for molecules, we ensure HierVAE can decode all the randomly selected encoded molecules. We split all the datasets into training and validation sets by 9:1.  

\textbf{Oracle details.} We evaluate three commonly used molecule optimization oracles including \textit{synthesis accessibility (SA)}, \textit{quantitative evaluation of drug-likeness (QED)} and activity again target \textit{GSK3B} from RDKit \citep{greg_landrum_2020_3732262} and TDC \citep{huang2021therapeutics}. All three oracles take as input a SMILES string representation of a molecule and return a scalar value of the property. The oracles are non-differentiable.

\textbf{Implementation details.}  We build the score network $s_\theta$ of the diffusion model using a 2-layer MLP architecture. This network features 1024 hidden dimensions and utilizes the ReLU activation function. The forward process adheres to a Variance Preserving (VP) SDE proposed by \citep{song2020score}. We calibrate the noise variance within this model, setting its minimum at $0.01$ and maximum at $2.0$.

For the surrogate model of the objective function, we use the ApproximateGP class in Gpytorch\footnote{\url{https://docs.gpytorch.ai/en/stable/_modules/gpytorch/models/approximate_gp.html\#ApproximateGP}}, which consists of a deep feature extractor and a Gaussian process. The feature extractor is a simple feed-forward network with two hidden layers, having widths of 500 and 50, respectively, and both employ ReLU activation functions. Regarding model optimization, we apply a fixed learning rate of $0.001$ for the diffusion model and $0.01$ for the surrogate model. Additionally, we set a batch size of 128 and conduct training over 1000 epochs for both models. For the sampling process, we use a consistent inverse temperature $\beta=10^4$ for all the three objectives. The step size $\Delta t$ is $0.001$ for the first stage, and $0.0001$ for the second stage.

 We sample $1000$ candidate solutions at test time for all the methods. For DDOM \citep{DDOM}, we use their implementation \footnote{\url{https://github.com/siddarthk97/ddom}}. For other baselines, we use the implementations provided by DesignBench\footnote{\url{https://github.com/brandontrabucco/design-bench}}. We tune the hyper-parameters of all the baselines as suggested in their papers.

For derivative-free optimization, the number of iterations is set to 100 for both the evolutionary algorithm and \ours. The number of particles for all methods remains the same as before, \ie~1000.

\begin{figure}[t!]
  \centering  
  \includegraphics[width=0.45\textwidth]{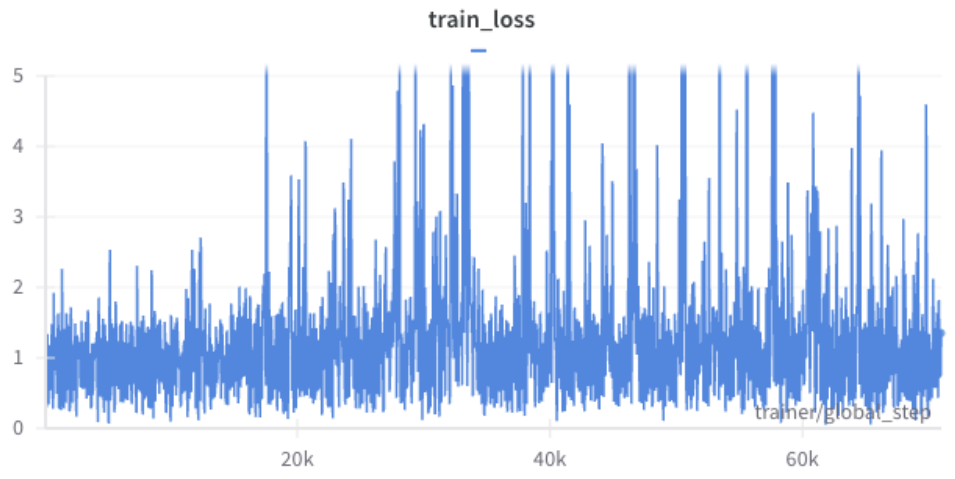}

  %\vspace{-1.5em}
  \caption{Training loss of the surrogate objective on Ant dataset}
  \label{fig:ant_train_loss}
  %\vspace{-1em}
\end{figure}

\section{Analysis on Ant Dataset}
\label{section:ant}

As shown in Table~\ref{tab:designbench}, \ours can only achieve subpar performance on the Ant dataset. This underperformance is primarily due to the difficulty in training the surrogate objective function. An illustration of this challenge is provided in Figure \ref{fig:ant_train_loss}, where the training loss of the surrogate objective is displayed. The training loss fluctuates throughout the training process. Although we investigated different network architectures, the issue still remains. Building an effective surrogate model for this dataset may require a more sophisticated architecture design.
We leave this for future work.

\section{Ablation Studies}
\label{sec:appendix:ablation}

\begin{table}[h]
\centering
% \begin{wraptable}{r}{0.55\textwidth}
\small
\resizebox{0.45\textwidth}{!}{

\begin{tabular}{@{}lcc@{}}
\toprule
                        & Superconductor & TFBind8 \\ \midrule
Best Baseline           &   $103.600\pm 8.139$             &    $0.981\pm 0.010$       \\
Only Stage I            &     $112.038\pm 6.783$           &  $0.984\pm 0.012 $         \\
Only Stage II           &    $92.432\pm 8.635$            &   $0.951\pm 0.028$        \\
Stage I + Stage II      &     $113.545\pm 5.322$           &         $0.987\pm 0.014$  \\
Stage I + Stage II + MH &         $\textbf{114.945}\pm 3.615$       &        $\textbf{0.989}\pm 0.021$   \\ \bottomrule
\end{tabular}}
\caption{Ablation study on the two-stage sampling.}
\label{table:ablation-two-stage}

% \end{wraptable}
\end{table}

\textbf{Impact of two-stage sampling.}
\cref{table:ablation-two-stage}  shows the impact of two-stage sampling on performance. Our findings reveal that even after the initial stage, \ours outperforms the top-performing baseline on both datasets. Relying solely on Langevin dynamics, without the warm-up phase of guided diffusion, results in significantly poorer results. This aligns with our discussion in \cref{sec:two-stage}, where we attributed this failure to factors such as the starting distribution, the schedule for step size adjustments, and the challenges posed by undefined gradients outside the feasible set. Integrating both stages yields a performance improvement as the initial stage can provide a better initialization within the data manifold  for the later stage (\cref{thm:concentration_warmup}). Adding the MH correction step further enhances results, leading to the best performance observed.

\textbf{Impact of annealing strategies.}  We study the influence of different annealing strategies for $\beta$ during the guided diffusion stage, focusing on the superconductor and TFBind8 datasets. We explore three strategies: constant, linear annealing, and exponential annealing. Figure~\ref{fig:annealing}(a) presents the performance across various diffusion steps. We find that our method is not particularly sensitive to the annealing strategies. However, it is worth noting that exponential annealing exhibits a marginal performance advantage over the others.

We also investigate how the value of $\beta$ at the end of annealing, denoted as $\beta_{\rm max}$, affects model performance in Figure~\ref{fig:annealing}(b). We find that increasing $\beta_{\rm max}$ initially leads to better performance. However, beyond a certain threshold, performance begins to decrease. It is noteworthy that the optimal value varies across different annealing strategies. Particularly, at $\beta_{\rm max} = 0$, the model reverts to a pure diffusion process, exhibiting the lowest performance due to the lack of guidance from the objective function.

\begin{figure}[t]
    \centering \includegraphics[width=0.95\textwidth]{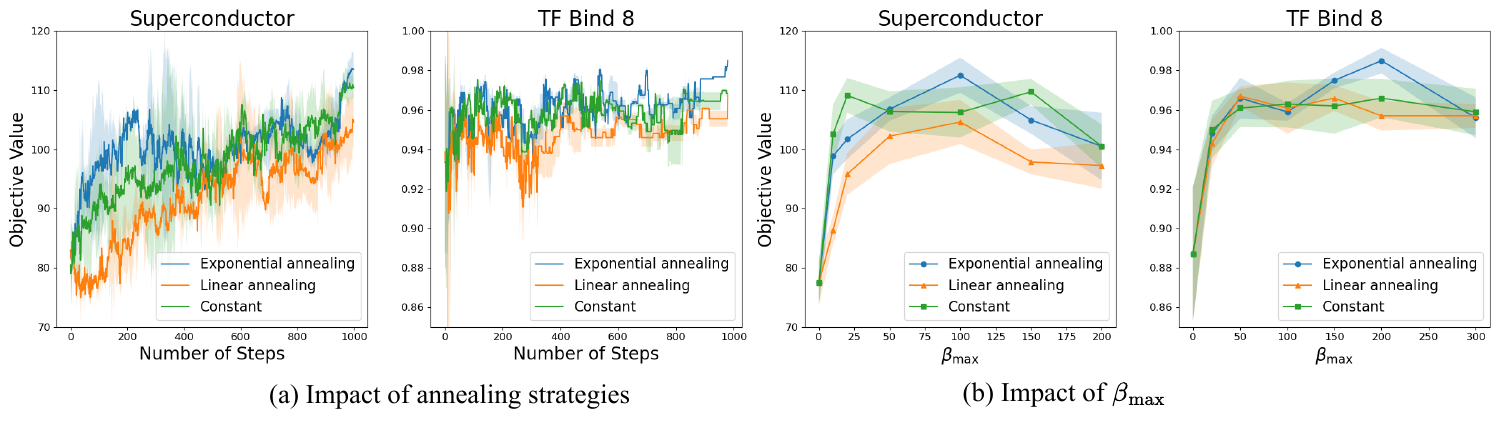}
    \caption{Impact of annealing strategies and $\beta_{\rm max}$ in the guided diffusion stage. $\beta_{\rm max}$ is the value of $\beta$ at the end of annealing.}
    %\vspace{-1.5em}
    \label{fig:annealing}
\end{figure}

\begin{figure}[t!]
   \centering

  %\vspace{cm}
  \includegraphics[width=0.95\textwidth]{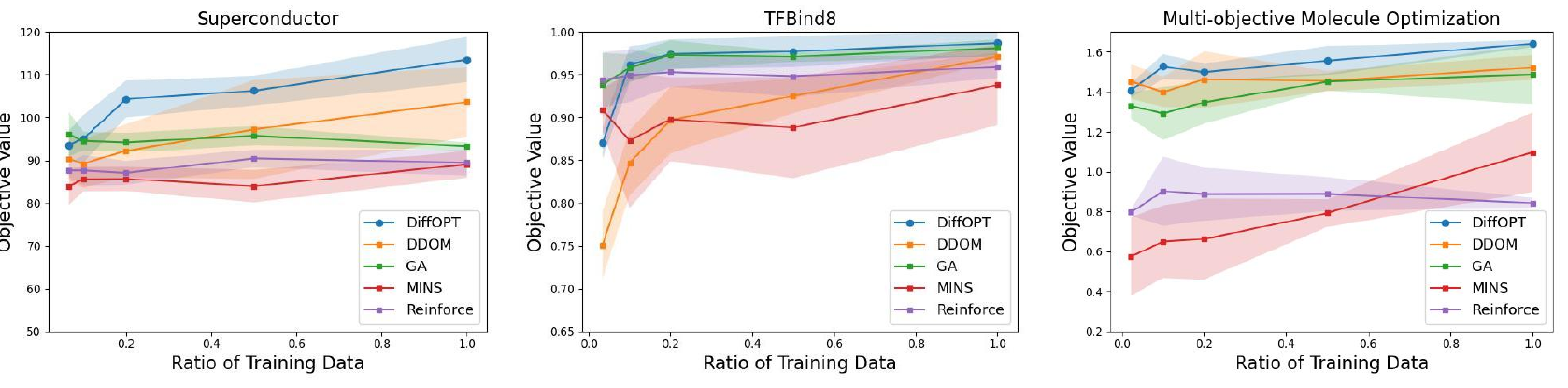}

  %\vspace{-1.5em}
  \caption{Impact of the number of training data.}
  \label{fig:training_data_size}
\end{figure}

\begin{figure}[t!]
   \centering

  %\vspace{cm}
  \includegraphics[width=0.95\textwidth]{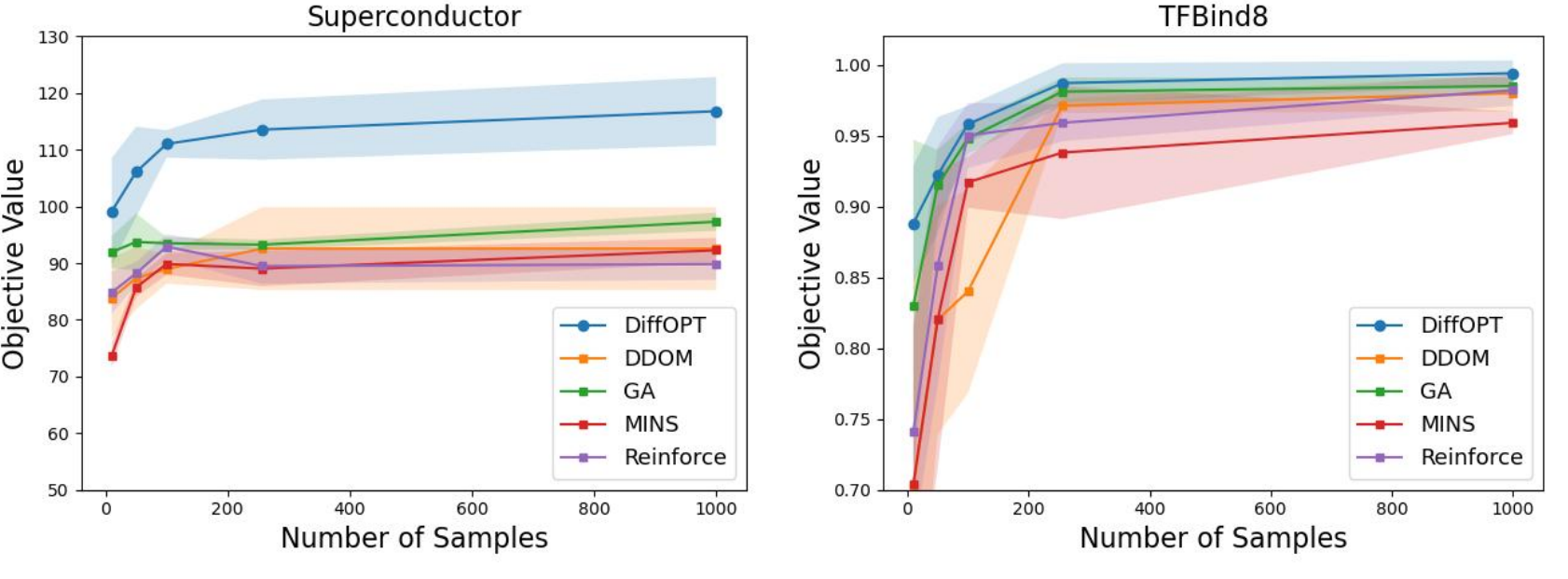}

  %\vspace{-1.5em}
  \caption{Impact of the number of samples at testing time.}
  \label{fig:training_sample}
\end{figure}

\paragraph{Sample efficiency}

We explore the sample efficiency of \ours at both training and testing stages. \Cref{fig:training_data_size} shows the performance of various methods versus the ratio of training data on Superconductor, TFBind8 and multi-objective molecule optimization. As we can see, on all the three datasets, DiffOPT can outperform all.

Our method is also sample efficient during inference. \Cref{fig:training_sample} shows the performance versus number of samples at inference stage. Notably, on both Superconductor and TFBind8, DiffOPT consistently outperforms all the baseline methods for various sample sizes during inference.

It is also important to highlight that our method consistently achieves much greater sample efficiency than DDOM at both training and inference stages, despite both approaches leveraging diffusion models.

% We finally investigate the impact of training data size on the superconductor dataset. We compare \ours with the best baseline on this dataset---DDOM. As we can see from Figure~\ref{fig:training_data_size}, \ours outperforms DDOM constantly across all the ratios of training data. 

\end{document}